\pdfoutput=1
\documentclass[11pt]{article}

\usepackage[margin=1in]{geometry}
\usepackage[numbers,sort&compress]{natbib}
\usepackage{amsmath,amssymb,amsthm}
\usepackage{bm}
\usepackage{algorithm}
\usepackage{algpseudocode}
\usepackage{authblk}

\usepackage{booktabs}
\usepackage{multirow}
\usepackage{graphicx}
\usepackage[table]{xcolor}
\usepackage{colortbl}
\usepackage{caption}
\usepackage{hyperref}
\usepackage{float}

\newtheorem{theorem}{Theorem}
\newtheorem{lemma}{Lemma}
\newtheorem{remark}{Remark}
\newtheorem{assumption}{Assumption}

\algrenewcommand{\algorithmiccomment}[1]{\hfill\textcolor{gray}{\(\triangleright\) \textit{#1}}}

\title{$Z^2$-Sampling: Zero-Cost Zigzag Trajectories for Semantic Alignment in Diffusion Models}
\author[1]{Haosen Li$^{*,}$}
\author[1]{Wenshuo Chen$^{*\dagger}$}
\author[1]{Shaofeng Liang}
\author[2]{Lei Wang}
\author[1]{Kaishen Yuan}
\author[1]{Yutao Yue$^{\dagger}$}

\affil[1]{The Hong Kong University of Science and Technology (Guangzhou)}
\affil[2]{Griffith University \& Data61/CSIRO}
\affil[]{\small $^{*}$Equal contribution.\quad $^{\dagger}$Corresponding author: wchen179@connect.hkust-gz.edu.cn, yutaoyue@hkust-gz.edu.cn}

\date{}

\begin{document}

\maketitle

\begin{figure}[H]
  \centering
  \includegraphics[width=\textwidth]{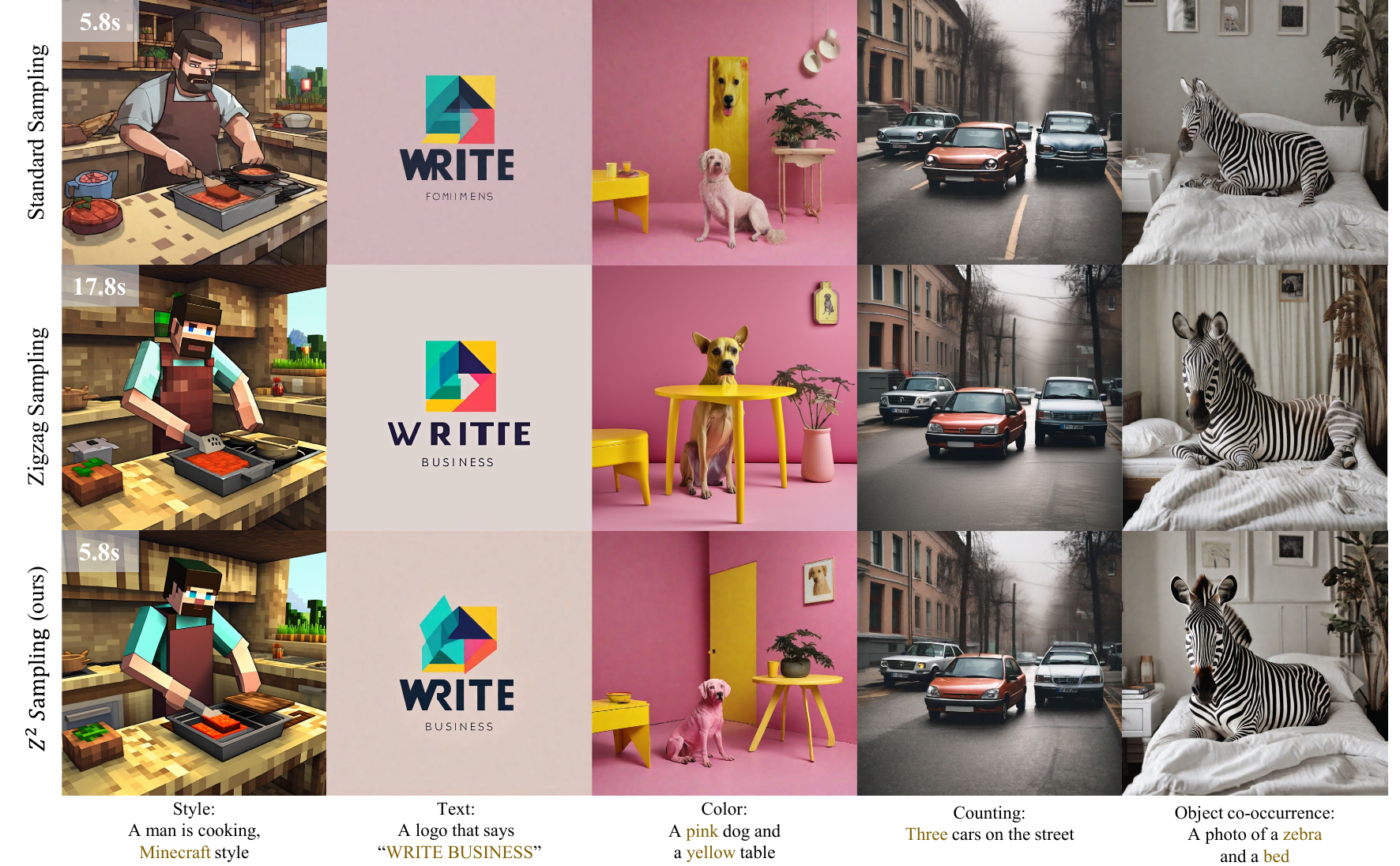} 
  \caption{Visualizing the structural paradigm shift: Explicit Z-Sampling (middle) incurs a 3$\times$ computational cost to physically traverse the intermediate states, suffering from off-manifold spatial approximation error $\tau(t)$. In contrast, $Z^2$-Sampling (bottom) algebraically collapses the zigzag trajectory into a single step, eliminating positional mismatch and reducing the cost back to the 2-NFE baseline via a zero-cost Temporal Semantic Surrogate.}
  \label{fig:teaser}
\end{figure}

\begin{abstract}

Diffusion models have achieved unprecedented success in text-aligned generation, largely driven by Classifier-Free Guidance (CFG). However, standard CFG operates strictly on instantaneous gradients, omitting the intrinsic curvature of the data manifold. Recent methods like Zigzag-sampling (Z-Sampling) explicitly traverse multi-step forward-backward trajectories to probe this curvature, significantly improving semantic alignment. Yet, these explicit traversals triple the Neural Function Evaluation (NFE) cost and introduce unconstrained truncation errors from off-manifold evaluations, causing cumulative drift from the true marginal distribution. 

In this paper, we theoretically demonstrate that the explicit zigzag sequence is topologically reducible. We propose \textit{Implicit Z-Sampling}, rigorously proving that intermediate states can be algebraically annihilated via operator dualities, physically eliminating off-manifold approximation errors. To push sampling efficiency to its theoretical lower bound, we introduce \textbf{$Z^2$-Sampling (Zero-cost Zigzag Sampling)}. Exploiting the Probability Flow ODE's temporal coherence, $Z^2$-Sampling couples implicit algebraic collapse with a dynamically cached \textit{Temporal Semantic Surrogate}. This restores the standard 2-NFE baseline without sacrificing semantic exploration. We formally prove via Backward Error Analysis that this discrete collapse inherently synthesizes a directional derivative curvature penalty. Finally, extensive evaluations demonstrate that $Z^2$-Sampling structurally shatters the performance-efficiency Pareto frontier. We validate its universal applicability across diverse architectures (U-Nets, DiTs) and modalities (image/video), establishing seamless orthogonality with advanced alignment frameworks (AYS, Diffusion-DPO).
\end{abstract}

\section{Introduction}
\label{sec:intro}

Diffusion models \cite{chen_2024, Chen_2025, chen2025polarisprojectionorthogonalsquaresrobust, yuan2025coemogensemanticallycoherentscalableemotional, ning2025dctdiffintriguingpropertiesimage, jia2025physicsinformedrepresentationalignmentsparse} have redefined the landscape of text-to-image ~\cite{ho2020denoisingdiffusionprobabilisticmodels, song2021scorebasedgenerativemodelingstochastic, song2022denoisingdiffusionimplicitmodels, lipman2023flowmatchinggenerativemodeling, liu2022flowstraightfastlearning, rombach2022highresolutionimagesynthesislatent, podell2023sdxlimprovinglatentdiffusion} and text-to-video \cite{blattmann2023stablevideodiffusionscaling, liu2024sorareviewbackgroundtechnology, singer2022makeavideotexttovideogenerationtextvideo, ho2022videodiffusionmodels} generation. The cornerstone of this success is Classifier-Free Guidance (CFG) \cite{ho2022classifierfreediffusionguidance}, which steers the generative trajectory toward text-aligned data manifolds by extrapolating score estimates. However, standard CFG is ``myopic'': it operates on the instantaneous gradient at the current latent state. Because the true data manifold of the Probability Flow ODE is highly non-linear \cite{chung2024cfgmanifoldconstrainedclassifierfree, karras2022elucidatingdesignspacediffusionbased}, localized linear extrapolation often struggles to foresee complex geometric shifts \cite{liu2023compositionalvisualgenerationcomposable}, leading to suboptimal alignment in dense compositional prompts.

To break this limitation, advanced sampling strategies, notably Zigzag Diffusion Sampling (Z-Sampling) \cite{bai2024zigzagdiffusionsamplingdiffusion} , oscillate the latent trajectory to probe local curvature \cite{song2021scorebasedgenerativemodelingstochastic}. By taking a forward step with a large guidance scale and a backward inversion step with a smaller scale, this lookahead-and-return mechanism \cite{lugmayr2022repaintinpaintingusingdenoising, xu2023restartsamplingimprovinggenerative} accumulates semantic foresight \cite{lu2022dpmsolverfastodesolver, Lu_2025} and enhances structural coherence.

However, explicitly traversing the latent space introduces two limitations. First is the computational overhead: physical steps require sequential network queries, tripling inference costs per step. Second, explicit Z-Sampling \cite{bai2024zigzagdiffusionsamplingdiffusion} queries the network from an intermediate, uncorrected state. Because large guidance scales push this intermediate state off the true data manifold \cite{chung2024cfgmanifoldconstrainedclassifierfree, ho2022classifierfreediffusionguidance, karras2022elucidatingdesignspacediffusionbased}, evaluating the network at this out-of-distribution point introduces a persistent spatial error. This error accumulates over time, pushing latents away from the marginal distribution and causing numerical drift, structural artifacts, and color oversaturation \cite{saharia2022photorealistictexttoimagediffusionmodels, lin2024commondiffusionnoiseschedules}.

Can we leverage the semantic benefits of zigzag exploration without incurring computational overhead and off-manifold drift? Inspired by recent advances in direct ODE inversion \cite{ju2023directinversionboostingdiffusionbased}—which recycle latent noise to mitigate inversion errors during image editing—we address this question through the principle of \textbf{Exact Noise Reuse}. We propose \textbf{Implicit Z-Sampling}, which, instead of evaluating the network at out-of-distribution intermediate states, deliberately recycles the unconditional noise evaluated at the exact on-manifold anchor. Crucially, anchoring the inversion to this reused noise analytically cancels out the intermediate physical steps via the intrinsic properties of deterministic ODE solvers. This analytical simplification collapses the two-step oscillation into a single deterministic translation, strictly eliminating the off-manifold approximation error.

To maximize sampling efficiency, we introduce \textbf{$Z^2$-Sampling (Zero-cost Zigzag Sampling)}. While Implicit Z-Sampling eliminates spatial error, $Z^2$-Sampling removes computational redundancy via \textit{temporal noise reuse}. Leveraging the temporal coherence of continuous diffusion trajectories—a property widely exploited in multi-step ODE solvers \cite{lu2022dpmsolverfastodesolver, zhao2023unipcunifiedpredictorcorrectorframework}—we design a dynamically cached \textit{Temporal Semantic Surrogate} that decouples spatial curvature evaluation from the main integration step. Substituting exact spatial noise differences with this cached proxy restores the standard baseline of one CFG evaluation (2 NFEs) per step, preserving non-linear semantic exploration. Furthermore, via Backward Error Analysis (BEA) \cite{Hairer2006}, we formally prove that $Z^2$-Sampling synthesizes a modified \textit{Effective Vector Field}. It naturally incorporates a directional derivative penalty that acts as a zero-cost geometric regularization, dynamically penalizing trajectory curvature.

Empirically, $Z^2$-Sampling analytically eliminates off-manifold degradation, structurally breaking the conventional performance-efficiency trade-off. Evaluated at its absolute minimum time budget ($\sim$6s), it decisively outperforms the peak generation quality of explicit Z-Sampling ($\sim$18s). Comprehensive evaluations demonstrate its universal applicability across diverse architectures, from standard U-Nets (SDXL \cite{podell2023sdxlimprovinglatentdiffusion}) to Diffusion Transformers (Hunyuan-DiT \cite{li2024hunyuandit}). Furthermore, it naturally extends to temporally coherent video synthesis (CogVideoX-2B \cite{yang2024cogvideox}, ModelScope-1.7B \cite{wang2023modelscope}) and establishes seamless orthogonal compatibility with both training-based (Diffusion-DPO \cite{wallace2023diffusionmodelalignmentusing}) and training-free (AYS \cite{sabour2024alignstepsoptimizingsampling}) alignment paradigms.

\section{Preliminaries and Problem Formulation}
\label{sec:preliminaries}

\textbf{First-Order Solvers \& CFG.} Deterministic solvers discretize the Probability Flow ODE into affine transitions. The forward step $\Phi^t$ and its exact inverse $\Psi^t$ map the latent state via a prediction $\bm{\epsilon}_\theta^t(\bm{x}_t)$:
\begin{align}
    \Phi^t(\bm{x}_t; \bm{\epsilon}) &\triangleq A_t \bm{x}_t + B_t \bm{\epsilon} = \bm{x}_{t-1} \label{eq:phi_fwd} \\
    \Psi^t(\bm{x}_{t-1}; \bm{\epsilon}) &\triangleq A_t^{-1} \bm{x}_{t-1} + C_t \bm{\epsilon} = \bm{x}_t \label{eq:psi_inv}
\end{align}
This unified formulation is quite versatile, as it naturally recovers several standard sampler geometries depending on the choice of coefficients. For instance, in the standard DDIM \cite{song2022denoisingdiffusionimplicitmodels} case ($\bm{\epsilon}$-prediction), we set $A_t = \sqrt{\alpha_{t-1}/\alpha_t}$ and $B_t = \sqrt{1-\alpha_{t-1}} - A_t \sqrt{1-\alpha_t}$. If we opt for the Euler method in velocity prediction \cite{lipman2023flowmatchinggenerativemodeling}, the rule reduces to a much simpler form with $A_t = 1$ and $B_t = \sigma_{t-1} - \sigma_t$. Furthermore, the framework extends to spherical geometries \cite{kingma2023variationaldiffusionmodels} ($\bm{v}$-prediction); by defining the angle as $\theta_t = \arccos(\sqrt{\alpha_t})$, the coefficients become purely trigonometric, yielding $A_t = \cos(\Delta\theta)$ and $B_t = \sin(\Delta\theta)$.

Standard CFG \cite{ho2022classifierfreediffusionguidance} extrapolates locally via $\bm{\epsilon}_\theta^t(\bm{x}_t, \gamma) = \bm{\epsilon}_\theta^t(\bm{x}_t, 0) + \gamma \Delta \bm{\epsilon}_\theta^t(\bm{x}_t)$, inherently ignoring manifold curvature.

\textbf{Z-Sampling \& Spatial Error.} \cite{bai2024zigzagdiffusionsamplingdiffusion} Explicit Z-Sampling probes curvature via a lookahead-and-return path:
\begin{align}
    \bm{x}_{t-1} &= \Phi^t\big(\bm{x}_t; \; \bm{\epsilon}_\theta^t(\bm{x}_t, \gamma_1)\big) \label{eq:z_fwd} \\
    \tilde{\bm{x}}_t &= \Psi^t\big(\bm{x}_{t-1}; \; \bm{\epsilon}_\theta^t(\bm{x}_{t-1}, \gamma_2)\big) \label{eq:z_inv}
\end{align}
While the original authors characterize an inherent inversion approximation error $\tau_2(t)$ within the score space, we identify a more fundamental geometric flaw. The large CFG forward step inherently pushes the intermediate state $\bm{x}_{t-1}$ off the true data manifold. Physically evaluating the network at this out-of-distribution state diverges from the pristine on-manifold unconditional noise $\bm{\epsilon}_\theta^t(\bm{x}_t, \gamma_2)$ (assuming $\gamma_2=0$ for standard inversion). This divergence translates into a persistent spatial error $\tau(t)$ upon inversion:
\begin{equation}
    \tau(t) = \Psi^t\big(\bm{x}_{t-1}; \; \bm{\epsilon}_\theta^t(\bm{x}_{t-1}, \gamma_2)\big) - \Psi^t\big(\bm{x}_{t-1}; \; \bm{\epsilon}_\theta^t(\bm{x}_t, \gamma_2)\big)
\end{equation}
\textit{Our $Z^2$-Sampling annihilates this positional mismatch ($\tau(t) \equiv 0$).}

\begin{figure*}[t]
\centering

\begin{minipage}[t]{0.44\textwidth}
    \vspace{0pt}
    \centering
    \includegraphics[width=\linewidth]{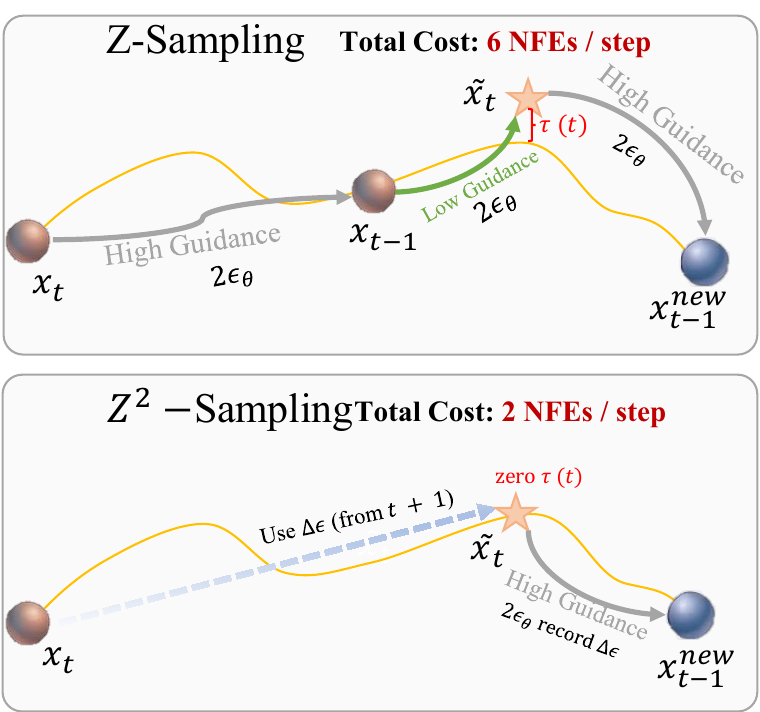}
\caption{\textbf{Methodological comparison of explicit Z-Sampling and $Z^2$-Sampling.}
\textbf{Top:} Explicit Z-Sampling traverses a lookahead-and-return path, requiring three CFG evaluations (6 NFEs per step).
The intermediate state is pushed off the true data manifold, introducing a persistent spatial approximation error $\tau(t)$ during inversion.
\textbf{Bottom:} $Z^2$-Sampling algebraically collapses the intermediate traversal and reuses a cached Temporal Semantic Surrogate ($\Delta \bm{\epsilon}$ from $t+1$).
This removes $\tau(t)$ and restores the standard 2-NFE computational cost.}
    \label{fig:method_z2}
\end{minipage}
\hfill
\begin{minipage}[t]{0.53\textwidth}
    \vspace{0pt}
    \captionsetup{type=algorithm}
    \caption{Z-Sampling versus $Z^2$-Sampling}
    \label{alg:comparison}
    \footnotesize

    \textbf{Z-Sampling}
    \begin{algorithmic}[1]
        \State \textbf{Input:} Fwd $\Phi^t$, Inv $\Psi^t$, Score $\bm{\epsilon}_\theta$, scales $\gamma_1, \gamma_2$, zigzag $\lambda$
        \State \textbf{Initialize:} $\bm{x}_T \sim \mathcal{N}(\mathbf{0}, \mathbf{I})$
        \For{$t = T$ \textbf{down to} $1$}
            \State $\bm{\epsilon}_t \gets \bm{\epsilon}_\theta^t(\bm{x}_t, \gamma_1)$ \Comment{\textbf{\textcolor{red}{(2 NFEs)}}}
            \State $\bm{x}_{t-1} \gets \Phi^t(\bm{x}_t ; \bm{\epsilon}_t)$ \Comment{\textit{1. First Denoising}}
            \If{$t > T - \lambda$} \Comment{\textit{Zigzag-Phase ($\lambda$ steps)}}
                \State $\bm{\epsilon}'_t \gets \bm{\epsilon}_\theta^t(\bm{x}_{t-1}, \gamma_2)$ \Comment{\textbf{\textcolor{red}{(2 NFEs)}}}
                \State $\tilde{\bm{x}}_t \gets \Psi^t(\bm{x}_{t-1}; \bm{\epsilon}'_t)$ \Comment{\textit{2. Inversion}}
                \State $\tilde{\bm{\epsilon}}_t \gets \bm{\epsilon}_\theta^t(\tilde{\bm{x}}_t, \gamma_1)$ \Comment{\textbf{\textcolor{red}{(2 NFEs)}}}
                \State $\bm{x}_{t-1}^{\mathrm{new}} \gets \Phi^t(\tilde{\bm{x}}_t; \tilde{\bm{\epsilon}}_t)$ \Comment{\textit{3. Re-denoising}}
                \State $\bm{x}_{t-1} \gets \bm{x}_{t-1}^{\mathrm{new}}$
            \EndIf
        \EndFor
        \State \textbf{return} $\bm{x}_0$
    \end{algorithmic}

    \vspace{0.8em}
    \hrule
    \vspace{0.8em}

    \textbf{$Z^2$-Sampling (Ours)}
    \begin{algorithmic}[1]
        \State \textbf{Input:} Fwd $\Phi^t$, Score $\bm{\epsilon}_\theta$, scale $\gamma_1$, warmup $W$, zigzag $\lambda$
        \State \textbf{Initialize:} $\bm{x}_T \sim \mathcal{N}(\mathbf{0}, \mathbf{I}), \ \Delta \bm{\epsilon} \gets \mathbf{0}$
        \For{$t = T$ \textbf{down to} $1$}
            \If{$T - W - \lambda < t \le T - W$} \Comment{\textit{Zigzag-Phase ($\lambda$ steps)}}
                \State $\tilde{\bm{x}}_t \gets \bm{x}_t - C_t \gamma_1 \Delta \bm{\epsilon}$
                \State $\tilde{\bm{\epsilon}}_t \gets \bm{\epsilon}_\theta^t(\tilde{\bm{x}}_t, \gamma_1)$ \Comment{\textbf{\textcolor{blue}{(2 NFEs)}}}
                \State $\bm{x}_{t-1}^{\mathrm{new}} \gets \Phi^t(\tilde{\bm{x}}_t; \tilde{\bm{\epsilon}}_t)$
                \State $\bm{x}_{t-1} \gets \bm{x}_{t-1}^{\mathrm{new}}$
                \State $\Delta \bm{\epsilon} \gets \Delta \bm{\epsilon}_\theta^t(\tilde{\bm{x}}_t)$ \Comment{\textbf{\textcolor{teal}{Update cache}}}
            \Else \Comment{\textit{Warmup \& Standard}}
                \State $\bm{\epsilon}_t \gets \bm{\epsilon}_\theta^t(\bm{x}_t, \gamma_1)$ \Comment{\textit{Standard CFG} \textbf{(2 NFEs)}}
                \State $\Delta \bm{\epsilon} \gets \Delta \bm{\epsilon}_\theta^t(\bm{x}_t)$ \Comment{\textbf{\textcolor{teal}{Init delta noise}}}
                \State $\bm{x}_{t-1} \gets \Phi^t(\bm{x}_t; \bm{\epsilon}_t)$ \Comment{\textit{Standard Denoising}}
            \EndIf
        \EndFor
        \State \textbf{return} $\bm{x}_0$
    \end{algorithmic}
\end{minipage}

\end{figure*}

\section{Methodology}
\label{sec:method}

\textbf{Roadmap.} We formalize our approach by: (1) Proving the algebraic collapse of spatial trajectories via exact noise reuse (Sec.~\ref{sec:implicit_z_sampling}); (2) Establishing \textbf{$Z^2$-Sampling} by bounding the Local Truncation Error of our Temporal Semantic Surrogate (Sec.~\ref{sec:z2_sampling}); and (3) Deriving the underlying Effective Vector Field to reveal the inherent continuous-time curvature penalty (Sec.~\ref{sec:generalization_ode}).

\subsection{Trajectory Collapse via Exact Noise Reuse}
\label{sec:implicit_z_sampling}

The fundamental flaw of explicit Z-Sampling lies in its spatial traversal: evaluating the inversion trajectory from the explicitly shifted intermediate state $\bm{x}_{t-1}$ forces the network to operate off-manifold, injecting an unconstrained approximation error $\tau(t)$.

To eradicate this deviation, we propose the principle of \textbf{Exact Noise Reuse}. Instead of moving to the uncharted state $\bm{x}_{t-1}$ to estimate inversion noise, we recycle the pristine unconditional noise $\bm{\epsilon}_\theta^t(\bm{x}_t, \gamma_2)$ evaluated at the on-manifold anchor $\bm{x}_t$. Anchoring the inversion to this reused noise circumvents the positional mismatch. 

This deliberate algorithmic choice of noise reuse triggers an elegant mathematical consequence: the algebraic dualities of the ODE solver completely annihilate the intermediate physical traversal.

Let $\bm{\epsilon}_\theta^t(\bm{x}, \gamma) \triangleq \bm{\epsilon}_\theta^t(\bm{x}, 0) + \gamma \Delta \bm{\epsilon}_\theta^t(\bm{x})$ and logically set $\gamma_2 = 0$. 

\begin{lemma}
    \label{lemma:involution}
  Deterministic affine solvers satisfy algebraic dualities:
    \begin{equation}
        \text{(i)} \;\; A_t^{-1} B_t \equiv -C_t \quad \text{and} \quad \text{(ii)} \;\; A_t C_t \equiv -B_t
    \end{equation}
\end{lemma}
\begin{proof}
  By definition, the exact inverse satisfies:
    \begin{align}
        \bm{x}_t &\equiv \Psi^t\big(\Phi^t(\bm{x}_t; \bm{\epsilon}); \bm{\epsilon}\big) \nonumber = A_t^{-1} \big( A_t \bm{x}_t + B_t \bm{\epsilon} \big) + C_t \bm{\epsilon} \nonumber \\
        &= \bm{x}_t + (A_t^{-1} B_t + C_t) \bm{\epsilon}
    \end{align}
    Since this holds $\forall \bm{\epsilon} \in \mathbb{R}^d$, we get $A_t^{-1} B_t + C_t = 0 \implies$ (i). Left-multiplying by $A_t$ yields (ii).
\end{proof}

\begin{remark}
    \label{rem:universal_c}
    This duality ($C_t \equiv -A_t^{-1}B_t$) is a topological constraint, unifying VP, Euler, and Spherical solvers without ad-hoc derivations.
\end{remark}

\begin{theorem}
    \label{thm:exact_semantic}
    Let $\bm{\epsilon}_1 = \bm{\epsilon}_\theta^t(\bm{x}_t, \gamma_1)$ be the forward guided noise, and \textbf{let $\bm{\epsilon}_2$ strictly reuse the unconditional noise evaluated at the anchor}: $\bm{\epsilon}_2 \triangleq \bm{\epsilon}_\theta^t(\bm{x}_t, \gamma_2)$. Assuming $\gamma_2=0$, under this exact noise reuse, the lookahead-and-return operator $\mathcal{C}^t(\bm{x}_t) \triangleq \Psi^t\big( \Phi^t(\bm{x}_t; \bm{\epsilon}_1); \bm{\epsilon}_2 \big)$ collapses into a deterministic single-step translation:
    \begin{equation}
        \tilde{\bm{x}}_t = \mathcal{C}^t(\bm{x}_t) = \bm{x}_t - C_t \gamma_1 \Delta \bm{\epsilon}_\theta^t(\bm{x}_t)
    \end{equation}
    Consequently, spatial approximation error strictly vanishes ($\tau(t) \equiv 0$).
\end{theorem}

\begin{proof}
    Expanding sequentially and invoking Lemma~\ref{lemma:involution}(i):
    \begin{align}
        \tilde{\bm{x}}_t &= A_t^{-1} \big( A_t \bm{x}_t + B_t \bm{\epsilon}_1 \big) + C_t \bm{\epsilon}_2 \nonumber \\
        &= \bm{x}_t + (A_t^{-1} B_t) \bm{\epsilon}_1 + C_t \bm{\epsilon}_2 \nonumber = \bm{x}_t - C_t (\bm{\epsilon}_1 - \bm{\epsilon}_2)
    \end{align}
  Substituting $\bm{\epsilon}_1 - \bm{\epsilon}_2 \equiv \gamma_1 \Delta \bm{\epsilon}_\theta^t(\bm{x}_t)$ completes the proof.
\end{proof}

\begin{remark}
    \textbf{Physical Essence vs. Mathematical Form.} Algebraic dualities act merely as mathematical vehicles. The zero spatial error ($\tau(t) \equiv 0$) fundamentally stems from \textbf{reusing the exact on-manifold noise}. This renders the intermediate state $\bm{x}_{t-1}$ a factored-out dummy variable, physically bypassing off-manifold errors rather than merely mitigating them.
\end{remark}

\begin{theorem}
    \label{thm:ultimate_collapse}
    The integration $\bm{x}_{t-1}^{\mathrm{new}} = \Phi^t(\tilde{\bm{x}}_t; \tilde{\bm{\epsilon}}_t)$, where $\tilde{\bm{\epsilon}}_t \triangleq \bm{\epsilon}_\theta^t(\tilde{\bm{x}}_t, \gamma_1)$, is isomorphic to a forward step initiated from $\bm{x}_t$:
    \begin{equation}
        \bm{x}_{t-1}^{\mathrm{new}} = \Phi^t\Big(\bm{x}_t; \; \tilde{\bm{\epsilon}}_t + \gamma_1 \Delta \bm{\epsilon}_\theta^t(\bm{x}_t) \Big)
    \end{equation}
\end{theorem}
\begin{proof}
    Expanding $\Phi^t$ and substituting $\tilde{\bm{x}}_t$ from Theorem~\ref{thm:exact_semantic}:
    \begin{align}
        \bm{x}_{t-1}^{\mathrm{new}} &= A_t \Big( \bm{x}_t - C_t \gamma_1 \Delta \bm{\epsilon}_\theta^t(\bm{x}_t) \Big) + B_t \tilde{\bm{\epsilon}}_t \nonumber \\
        &= A_t \bm{x}_t - (A_t C_t) \gamma_1 \Delta \bm{\epsilon}_\theta^t(\bm{x}_t) + B_t \tilde{\bm{\epsilon}}_t
    \end{align}
    By Lemma~\ref{lemma:involution}(ii), $A_t C_t = -B_t$. Factoring out $B_t$ yields:
    \begin{align}
        \bm{x}_{t-1}^{\mathrm{new}} &= A_t \bm{x}_t + B_t \Big( \tilde{\bm{\epsilon}}_t + \gamma_1 \Delta \bm{\epsilon}_\theta^t(\bm{x}_t) \Big)
    \end{align}
\end{proof}

\begin{figure}[t] 
    \centering
    \includegraphics[width=\linewidth]{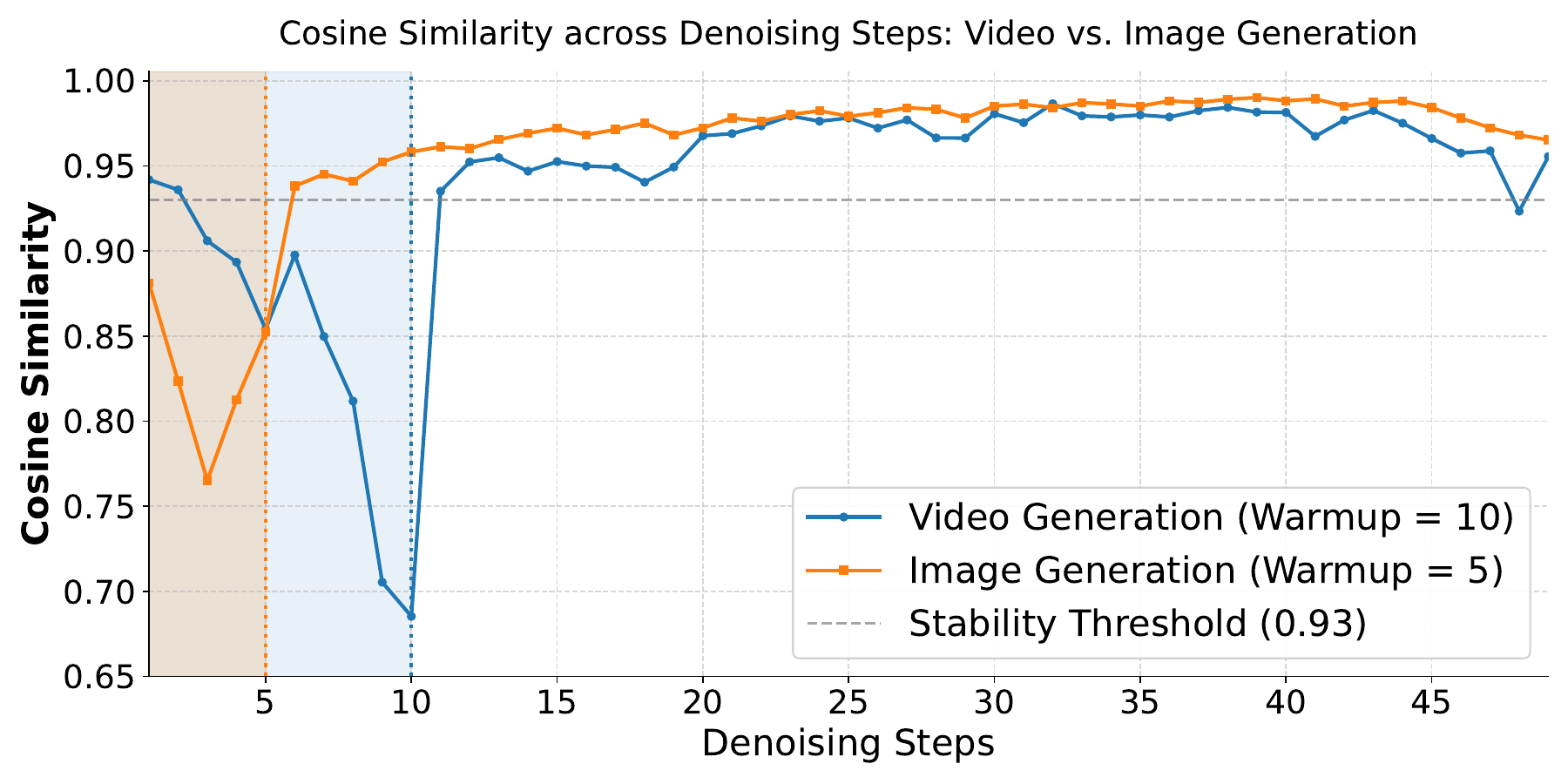}
    \caption{\textbf{Cosine similarity across denoising steps.} Empirical evaluation across a large number of samples reveals a consistent pattern: video generation requires a 10-step warmup phase to achieve semantic stability, whereas image generation converges rapidly in just 5 steps.}
    \label{fig:warmup}
\end{figure}

\subsection{Temporal Surrogate and LTE Bounds}
\label{sec:z2_sampling}

To avoid the overhead of querying $\bm{x}_t$ while preserving zero spatial error ($\tau(t) \equiv 0$, Theorem~\ref{thm:exact_semantic}), \textbf{$Z^2$-Sampling} caches a \textbf{Temporal Semantic Surrogate} to maintain the 2-NFE baseline (Algorithm~\ref{alg:comparison}):
\begin{equation}
    \Delta \bm{\epsilon} \triangleq \Delta \bm{\epsilon}_\theta^{t+1}(\tilde{\bm{x}}_{t+1}) \approx \Delta \bm{\epsilon}_\theta^t(\bm{x}_t)
\end{equation}
Exploiting rapid temporal stabilization (Figure~\ref{fig:warmup}), this surrogate uses a brief caching warmup to eliminate redundant queries, yielding a Local Truncation Error of $B_t \sim \mathcal{O}(h)$.

\begin{assumption}
    \label{assum:lipschitz}
    Let $\mathcal{M} \subset \mathcal{X}$ be the data manifold. The operator $\Delta \bm{\epsilon}_\theta^t(\bm{x})$ is uniformly Lipschitz continuous ($L_x, L_t > 0$) strictly within a local neighborhood $\mathcal{T}_\delta(\mathcal{M})$. Beyond this, $L_{\mathrm{extruded}} \gg L_x$.
\end{assumption}

\begin{theorem}
    \label{thm:lte_analysis}
    Unlike explicit Z-Sampling which suffers an unconstrained $\mathcal{O}(1)$ persistent spatial bias off-manifold, $Z^2$-Sampling bounds the surrogate error by $\mathcal{O}(h)$, yielding an $\mathcal{O}(h^2)$ LTE and preserving convergence order.
\end{theorem}

\begin{proof}
    \textbf{Explicit Z-Sampling:} Let $\bm{\epsilon}_\theta^t(\bm{x}_{t-1}, \gamma_2) \approx \bm{\epsilon}_\theta^t(\bm{x}_t, \gamma_2)$\\ $\bm{x}_{t-1} \notin \mathcal{T}_\delta(\mathcal{M})$. Since $L_{\mathrm{extruded}} \to \infty$, $\|\tau(t)\|$ injects an $\mathcal{O}(1)$ bias.
    
    \textbf{$Z^2$-Sampling:} The surrogate discrepancy is $\mathcal{E}_{\mathrm{TSS}} = \|\Delta \bm{\epsilon}_\theta^t(\bm{x}_t) - \Delta \bm{\epsilon}\|$. Invoking the Triangle Inequality within $\mathcal{T}_\delta(\mathcal{M})$:
    \begin{align}
        \mathcal{E}_{\mathrm{TSS}} &\le \|\Delta \bm{\epsilon}_\theta^t(\bm{x}_t) - \Delta \bm{\epsilon}_\theta^{t+1}(\bm{x}_t)\| + \|\Delta \bm{\epsilon}_\theta^{t+1}(\bm{x}_t) - \Delta \bm{\epsilon}_\theta^{t+1}(\tilde{\bm{x}}_{t+1})\| \nonumber \\
        &\le L_t h + L_x \|\bm{x}_t - \tilde{\bm{x}}_{t+1}\|
    \end{align}
Bounding the ODE step displacement by $\|\bm{x}_t - \tilde{\bm{x}}_{t+1}\| \le Kh$ yields $\mathcal{E}_{\mathrm{TSS}} = \mathcal{O}(h)$. Modulated by $B_t \sim \mathcal{O}(h)$, the cumulative LTE is:
    \begin{equation}
        \mathrm{LTE} = B_t \cdot \mathcal{E}_{\mathrm{TSS}} = \mathcal{O}(h) \cdot \mathcal{O}(h) = \mathcal{O}(h^2)
    \end{equation}
\end{proof}

\subsection{Backward Error Analysis}
\label{sec:generalization_ode}

Applying Backward Error Analysis yields the \textit{Modified Equation}.

\begin{theorem}
    \label{thm:modified_equation}
    For a CFG-guided ODE $\frac{\mathrm{d}\bm{x}}{\mathrm{d}t} = \bm{v}_{uc}(\bm{x}) + \gamma_1 \Delta \bm{v}(\bm{x})$, integrating the $Z^2$-Sampling sequence solves a modified initial value problem governed by an \textbf{Effective Vector Field}:
    \begin{equation*}
        \bm{v}_{\mathrm{eff}}(\bm{x}) = \bm{v}_{uc}(\bm{x}) + 2\gamma_1 \Delta \bm{v}(\bm{x}) - h \gamma_1 \nabla_{\Delta \bm{v}} \bm{v}(\bm{x}, \gamma_1) + \mathcal{O}(h^2)
    \end{equation*}
    where $\nabla_{\Delta \bm{v}} \bm{v} \triangleq \mathbf{J}_{\bm{v}} \Delta \bm{v}$ represents the \textbf{Directional Derivative}.
\end{theorem}

\begin{proof}
    Mapping to an explicit Euler discretization ($A_t \to 1, B_t \to -h, C_t \to h$), the intermediate state is $\tilde{\bm{x}}_t = \bm{x}_t - h \gamma_1 \Delta \bm{v}(\bm{x}_t)$. Taylor expanding $\tilde{\bm{v}}_t \triangleq \bm{v}(\tilde{\bm{x}}_t, \gamma_1)$ around $\bm{x}_t$:
    \begin{align}
        \tilde{\bm{v}}_t &= \bm{v}(\bm{x}_t, \gamma_1) - h \gamma_1 \mathbf{J}_{\bm{v}}(\bm{x}_t, \gamma_1) \Delta \bm{v}(\bm{x}_t) + \mathcal{O}(h^2) \nonumber \\
        &= \bm{v}(\bm{x}_t, \gamma_1) - h \gamma_1 \nabla_{\Delta \bm{v}} \bm{v}(\bm{x}_t, \gamma_1) + \mathcal{O}(h^2)
    \end{align}
    The secondary integration $\bm{x}_{t-1}^{\mathrm{new}} = \bm{x}_t - h \big( \tilde{\bm{v}}_t + \gamma_1 \Delta \bm{v}(\bm{x}_t) \big)$ becomes:
    \begin{align*}
    \bm{x}_{t-1}^{\mathrm{new}} &= \bm{x}_t - h \Big[ \bm{v}(\bm{x}_t, \gamma_1) + \gamma_1 \Delta \bm{v}(\bm{x}_t) - h \gamma_1 \nabla_{\Delta \bm{v}} \bm{v}(\bm{x}_t, \gamma_1) \Big] + \mathcal{O}(h^3)
    \end{align*}
    With $\bm{v}(\bm{x}_t, \gamma_1) = \bm{v}_{uc}(\bm{x}_t) + \gamma_1 \Delta \bm{v}(\bm{x}_t)$, the difference quotient is:
    \begin{align*}
    \frac{\bm{x}_t - \bm{x}_{t-1}^{\mathrm{new}}}{h} &= \bm{v}_{uc}(\bm{x}_t) + 2\gamma_1 \Delta \bm{v}(\bm{x}_t) - h \gamma_1 \nabla_{\Delta \bm{v}} \bm{v}(\bm{x}_t, \gamma_1) + \mathcal{O}(h^2)
    \end{align*}
Via Backward Error Analysis, this integration equates to a single Euler step ($\bm{x}_{t-1}^{\mathrm{new}} = \bm{x}_t - h \bm{v}_{\mathrm{eff}}(\bm{x}_t)$) with the \textbf{Effective Vector Field}:
    \begin{equation*}
        \bm{v}_{\mathrm{eff}}(\bm{x}) = \bm{v}_{uc}(\bm{x}) + 2\gamma_1 \Delta \bm{v}(\bm{x}) - h \gamma_1 \nabla_{\Delta \bm{v}} \bm{v}(\bm{x}, \gamma_1) + \mathcal{O}(h^2)
    \end{equation*}
\end{proof}

\begin{remark}
Theorem~\ref{thm:modified_equation} shows $Z^2$-Sampling implicitly doubles guidance ($2\gamma_1$). While standard CFG collapses at ultra-high scales, explicit Z-Sampling survives due to a curvature penalty ($-h \gamma_1 \nabla_{\Delta \bm{v}} \bm{v}$) acting as a Riemannian centripetal force against off-manifold escapes. By algebraically internalizing this exact regularization, $Z^2$-Sampling bypasses explicit Z-Sampling's severe overhead and spatial drift, achieving zero extra cost and $\tau(t) \equiv 0$ (shown in Appendix B).
\end{remark}

\section{Empirical Analysis}

\subsection{Experimental Settings}
\textbf{Datasets \& Metrics}\quad
We evaluate on three text-to-image (Pick-a-Pic \cite{kirstain2023pickapicopendatasetuser}, DrawBench \cite{saharia2022photorealistictexttoimagediffusionmodels}, GenEval \cite{ghosh2023genevalobjectfocusedframeworkevaluating}) and one text-to-video (ChronoMagic-Bench-150 \cite{yuan2024chronomagic}) benchmarks. For images, we use HPS v2 \cite{wu2023human}, PickScore \cite{kirstain2023pickapicopendatasetuser}, ImageReward (IR) \cite{xu2023imagerewardlearningevaluatinghuman}, and Aesthetic Score (AES) \cite{schuhmann2022laion5bopenlargescaledataset}. For videos, temporal consistency and prompt fidelity are measured via CHScore Flow \cite{yuan2024chronomagic}, CLIP SIM \cite{radford2021learningtransferablevisualmodels}, Frame LPIPS \cite{zhang2018perceptual}/SSIM \cite{nilsson2020understandingssim}, and MTScore CLIP.

\noindent\textbf{Models \& Settings}\quad 
We test on standard U-Nets (SD-2.1 \cite{rombach2022highresolutionimagesynthesislatent}, SDXL \cite{podell2023sdxlimprovinglatentdiffusion}, ModelScope-1.7B \cite{wang2023modelscope}) and DiTs (Hunyuan-DiT \cite{li2024hunyuandit}, CogVideoX-2B \cite{yang2024cogvideox}), where we perform 50 denoising steps ($T=50$). For the distilled DreamShaper-xl-v2-turbo \cite{sauer2023adversarialdiffusiondistillation}, we set the denoising step $T$ only to 4 via the Euler sampler. We set the guidance scale $\gamma_1 = 5.5$ for SD-2.1 and SDXL, $\gamma_1 = 6.0$ for Hunyuan-DiT,

\begin{table}[H]
  \centering
  \caption{Quantitative comparison: $\mathbf{Z^2}$\textbf{-Sampling} matches or exceeds Z-Sampling's performance while being $\sim$3$\times$ faster.}
  \label{tab:main_comparison}
  \resizebox{\textwidth}{!}{
  \begin{tabular}{llc@{\hspace{1.5em}}cccc@{\hspace{1.5em}}cccc}
    \toprule
    \multicolumn{2}{c}{\multirow{2}{*}{\textbf{Method}}} & \multirow{2}{*}{\textbf{Time (s) $\downarrow$}} & \multicolumn{4}{c}{\textbf{Pick-a-Pic}} & \multicolumn{4}{c}{\textbf{DrawBench}} \\
    \cmidrule(lr){4-7} \cmidrule(lr){8-11}
    & & & HPS v2 $\uparrow$ & AES $\uparrow$ & PickScore $\uparrow$ & IR $\uparrow$ & HPS v2 $\uparrow$ & AES $\uparrow$ & PickScore $\uparrow$ & IR $\uparrow$ \\
    \midrule
    
    \multirow{3}{*}{SD-2.1} 
    & Standard & 3.0 & 23.05 & 5.28 & 19.08 & -43.66 & 23.90 & 5.20 & 20.49 & -44.34 \\
    & Z-Sampling & 9.5 & 24.53 & \textbf{5.47} & 19.51 & -18.62 & \textbf{24.67} & 5.29 & \textbf{20.82} & -23.61 \\
    & $Z^2$-Sampling (ours) & \textbf{3.0} & \textbf{24.58} & 5.45 & \textbf{19.54} & \textbf{-18.45} & 24.60 & \textbf{5.31} & 20.78 & \textbf{-23.40} \\
    \midrule

    \multirow{3}{*}{SDXL} 
    & Standard & 5.8 & 29.89 & 6.09 & 21.63 & 58.65 & 28.81 & 5.56 & 22.31 & 60.75 \\
    & Z-Sampling & 17.8 & \textbf{31.28} & 6.13 & 21.85 & \textbf{79.22} & 30.50 & 5.68 & \textbf{22.46} & 79.97 \\
    & $Z^2$-Sampling (ours) & \textbf{5.8} & 31.11 & \textbf{6.16} & \textbf{21.88} & 76.63 & \textbf{30.55} & \textbf{5.71} & 22.39 & \textbf{80.12} \\
    \midrule

    \multirow{3}{*}{\begin{tabular}[c]{@{}l@{}}DreamShaper\\-xl-v2-turbo\end{tabular}} 
    & Standard & 1.4 & 30.04 & 5.93 & 21.59 & 66.18 & 26.85 & 5.28 & 21.77 & 40.22 \\
    & Z-Sampling & 4.5 & 32.38 & \textbf{6.15} & 22.11 & 90.87 & \textbf{29.90} & 5.64 & \textbf{22.35} & 73.51 \\
    & $Z^2$-Sampling (ours) & \textbf{1.4} & \textbf{32.42} & 6.10 & \textbf{22.14} & \textbf{91.05} & 29.81 & \textbf{5.67} & 22.30 & \textbf{73.80} \\
    \midrule

    \multirow{3}{*}{Hunyuan-DiT} 
    & Standard & 60.8 & 30.82 & 6.20 & 21.88 & 94.22 & 30.22 & 5.70 & 22.29 & 82.63 \\
    & Z-Sampling & 182.3 & \textbf{31.12} & 6.31 & \textbf{21.90} & 97.88 & 30.53 & 5.75 & \textbf{22.40} & 96.13 \\
    & $Z^2$-Sampling (ours) & \textbf{60.8} & 31.05 & \textbf{6.34} & 21.87 & \textbf{98.15} & \textbf{30.58} & \textbf{5.78} & 22.36 & \textbf{96.32} \\
    \bottomrule
  \end{tabular}
  }
\end{table}

\begin{table}[H]
  \centering
  \caption{Enhancing the training-free AYS (DreamShaper-xl-v2-turbo) and training-based Diffusion-DPO (SDXL).}
  \label{tab:combined_enhancement}
  \resizebox{\textwidth}{!}{
  \begin{tabular}{ll@{\hspace{1.5em}}cccc@{\hspace{2em}}cccc}
    \toprule
    \multirow{2}{*}{\textbf{Model}} & \multirow{2}{*}{\textbf{Method}} & \multicolumn{4}{c}{\textbf{Pick-a-Pic}} & \multicolumn{4}{c}{\textbf{DrawBench}} \\
    \cmidrule(lr){3-6} \cmidrule(lr){7-10}
    & & HPS v2$\uparrow$ & AES$\uparrow$ & PickScore$\uparrow$ & IR$\uparrow$ & HPS v2$\uparrow$ & AES$\uparrow$ & PickScore$\uparrow$ & IR$\uparrow$ \\
    \midrule
    \multirow{3}{*}{AYS} & Baseline & 32.78 & 6.05 & 22.32 & 91.88 & 30.95 & 5.57 & 22.68 & 77.85 \\
    & + Z-Sampling  & \textbf{33.57} & 6.15 & 22.45 & \textbf{104.22} & 31.93 & \textbf{5.72} & 22.75 & 94.82 \\
    & + $Z^2$-Sampling (ours) & 33.52 & \textbf{6.18} & \textbf{22.49} & 103.90 & \textbf{31.98} & 5.70 & \textbf{22.79} & \textbf{95.10} \\
    \midrule
    \multirow{3}{*}{Diffusion-DPO} & Baseline & 31.41 & 5.60 & 22.00 & 90.28 & 29.80 & 5.66 & 22.47 & 85.94 \\
    & + Z-Sampling & 31.60 & \textbf{6.08} & 22.18 & 94.48 & \textbf{30.35} & 5.67 & 22.47 & \textbf{93.34} \\
    & + $Z^2$-Sampling (ours) & \textbf{31.65} & 6.05 & \textbf{22.22} & \textbf{94.80} & 30.28 & \textbf{5.71} & \textbf{22.50} & 93.10 \\
    \bottomrule
  \end{tabular}
  }
\end{table}

\begin{table}[H]
  \caption{Quantitative results on GenEval (SDXL). All values are percentages (\%).}
  \label{tab:geneval_results}
  \resizebox{\linewidth}{!}{
  \begin{tabular}{lccccccc}
    \toprule
    \textbf{Method} & \textbf{Single Object} $\uparrow$ & \textbf{Two Objects} $\uparrow$ & \textbf{Counting} $\uparrow$ & \textbf{Colors} $\uparrow$ & \textbf{Position} $\uparrow$ & \textbf{Color Attribution} $\uparrow$ & \textbf{Average} $\uparrow$ \\
    \midrule
    Standard    & 97.50 & 69.70 & 33.75 & 86.71 & 10.00 & 18.00 & 52.52 \\
    Z-Sampling     & \textbf{100.0} & 74.75 & 46.25 & \textbf{87.23} & 10.00 & 24.00 & 57.04 \\
    $Z^2$-Sampling & 99.50 & \textbf{75.10} & \textbf{46.80} & 87.05 & \textbf{12.00} & \textbf{24.50} & \textbf{57.35} \\
    \bottomrule
  \end{tabular}
  }
\end{table}

ModelScope-1.7B, and CogVideoX-2B, and $\gamma_1 = 3.5$ for DreamShaper-xl-v2-turbo. The inversion guidance scale is set to $\gamma_2 = 0$. Zigzag operations are applied strictly after stabilization, spanning the remaining $\lambda = T - W - 1$ steps, where $W$ denotes the warmup duration. Specifically, we set $W=5$ for standard image generation models, $W=10$ for video generation models, and scale it down to a single step ($W=1$) for distilled regimes (DreamShaper).

\noindent\textbf{Baselines}\quad
We compare $Z^2$-Sampling against Standard sampling, AYS (training-free), Diffusion DPO (training-based), and Z-Sampling.

\begin{figure*}[t]
    \centering
    \includegraphics[width=\linewidth]{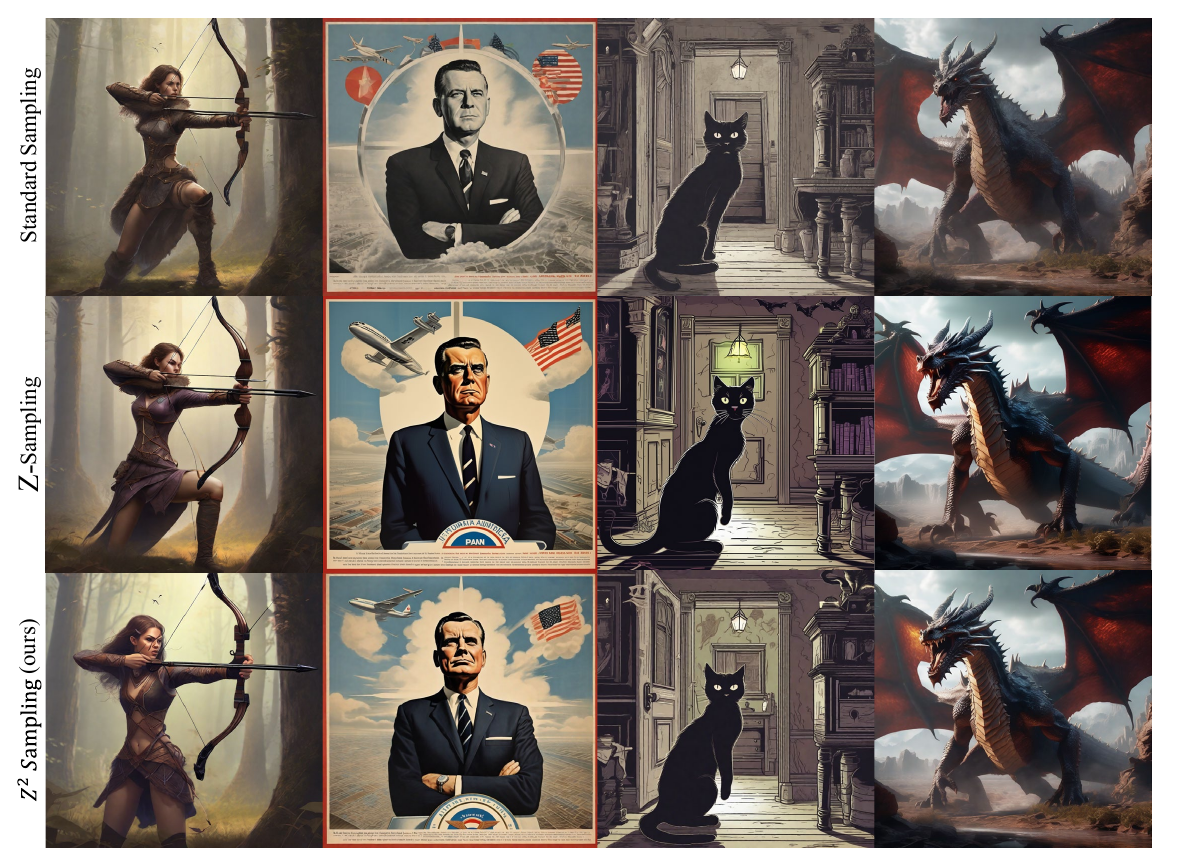}
    \caption{Qualitative comparison of text-to-image generation. From top to bottom: Standard Sampling, Z-Sampling, and our proposed $Z^2$-Sampling. As illustrated, $Z^2$-Sampling accumulates richer semantic information compared to Standard Sampling, while effectively suppressing the color oversaturation artifacts present in Z-Sampling.}
    \label{fig:placeholder}
\end{figure*}

\subsection{Text-to-Image Generation}

\noindent\paragraph{Quantitative Analysis.}
As shown in Table~\ref{tab:main_comparison}, standard CFG struggles with semantic alignment. While explicit Z-Sampling significantly improves prompt adherence, it triples inference latency (e.g., 5.8s to 17.8s for SDXL). In contrast, $Z^2$-Sampling restores the standard 2-NFE baseline speed while matching or exceeding the semantic gains of Z-Sampling (e.g., achieving a PickScore of 21.88 on Pick-a-Pic vs. Z-Sampling's 21.85). This empirically confirms that our implicit algebraic collapse captures manifold curvature without the redundant, expensive physical traversals.

\noindent\paragraph{Fine-Grained Compositional Generation.}
On the GenEval benchmark (Table~\ref{tab:geneval_results}), $Z^2$-Sampling achieves a superior overall accuracy of 57.35\%. It particularly excels in complex relational tasks, outperforming explicit Z-Sampling in ``Two object'' (75.10\% vs. 74.75\%) and ``Counting'' (46.80\% vs. 46.25\%). This validates that our Temporal Semantic Surrogate successfully preserves non-linear semantic exploration while strictly avoiding the spatial off-manifold errors that typically cause object blending and attribute leakage.

\noindent\paragraph{Orthogonal Compatibility.}
Table~\ref{tab:combined_enhancement} demonstrates the plug-and-play versatility of $Z^2$-Sampling. Integrating it into the training-free AYS framework yields compounded preference gains, boosting ImageReward to 103.90. Similarly, applying it to training-based Diffusion-DPO improves PickScore to 22.22 without adding inference latency. This proves that our implicit curvature penalty ($\nabla_{\Delta \bm{v}} \bm{v}$) is orthogonal to both step-distillation schedules and weight-level fine-tuning.

\noindent\paragraph{Qualitative Analysis.}
Qualitatively (Figure~\ref{fig:placeholder}), standard sampling omits elements in dense prompts. While explicit Z-Sampling improves semantic richness, its off-manifold evaluations induce a persistent numerical drift $\tau(t)$, manifesting as structural artifacts and color oversaturation. By enforcing $\tau(t) \equiv 0$ through exact noise reuse, $Z^2$-Sampling fundamentally suppresses this degradation. It yields text-aligned, photorealistic images while bypassing the visual artifacts normally associated with high-scale guidance.

\subsection{Extension to Text-to-Video Generation}

\textbf{The Spatiotemporal Bottleneck.} Transitioning to video synthesis introduces the strict requirement of temporal consistency, rendering standard Z-Sampling infeasible. Its physical off-manifold traversals not only triple the inference time but also introduce severe numerical drift. When propagated across the temporal axis,

these spatial errors compound exponentially, aggressively destroying inter-frame continuity and leading to structural collapse.

\textbf{Zero-Cost Generalization via Algebraic Collapse.} In contrast, our proposed $Z^2$-Sampling provides a seamless, zero-cost extension. By employing an algebraic formulation that physically bypasses off-manifold errors, temporal coherence is strictly preserved. Our cached Temporal Semantic Surrogate naturally adapts to spatiotemporal attention mechanisms without any architectural modifications. Evaluated on fundamentally different backbones (DiT-based CogVideoX-2B and U-Net-based ModelScope-1.7B), our method consistently outperforms standard baselines without introducing computational overhead (Table~\ref{tab:evaluation_results}).

\begin{figure}[t]
    \centering
    \includegraphics[width=\linewidth]{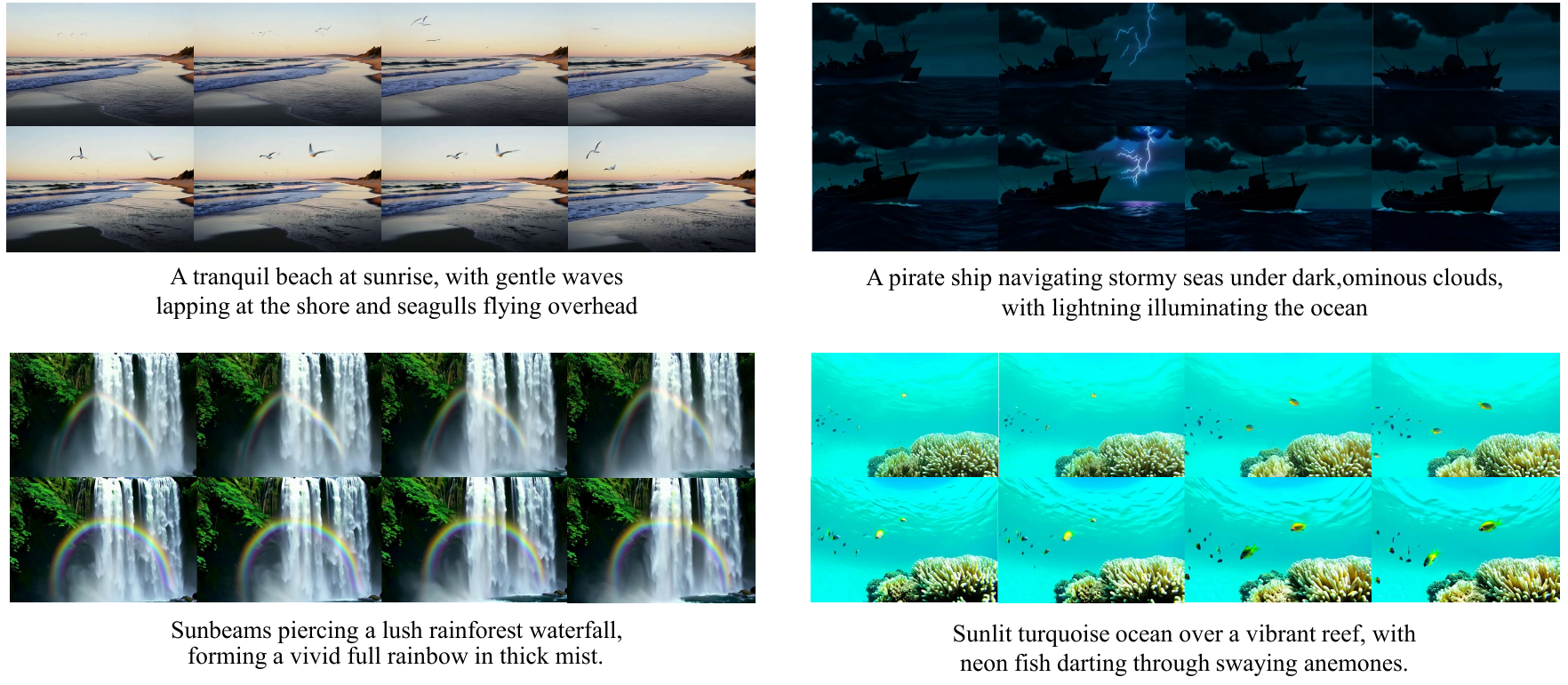}
   \caption{Qualitative comparison of video generation. In each scenario, the top row displays the Standard sampling baseline, and the bottom row shows our proposed $Z^2$-Sampling. As illustrated, $Z^2$-Sampling achieves superior semantic alignment with the text prompts and significantly stronger temporal consistency.}
    \label{fig:video}
\end{figure}

\begin{table}[t]
  \centering
  \caption{Quantitative evaluation on ChronoMagic-Bench-150 (CogVideoX-2B and ModelScope-1.7B).}
  \label{tab:evaluation_results}
  \resizebox{\textwidth}{!}{
  \begin{tabular}{llccccc}
    \toprule
    \textbf{Model} & \textbf{Method} & \textbf{CHScore Flow}$\uparrow$ & \textbf{CLIP SIM}$\uparrow$ & \textbf{Frame LPIPS}$\downarrow$ & \textbf{Frame SSIM}$\uparrow$ & \textbf{MTScore CLIP}$\uparrow$ \\
    \midrule
    \multirow{2}{*}{CogVideoX-2B} & Standard       & 59.93 & 29.00 & 9.68 & 87.07 & 5.38 \\
                                  & $Z^2$-Sampling & \textbf{60.24} & \textbf{29.25} & \textbf{9.24} & \textbf{87.61} & \textbf{5.47} \\
    \midrule
    \multirow{2}{*}{ModelScope-1.7B} & Standard       & 70.47 & 27.79 & 8.12 & 79.13 & 10.56 \\
                                     & $Z^2$-Sampling & \textbf{72.88} & \textbf{28.03} & \textbf{7.04} & \textbf{81.03} & \textbf{10.89} \\
    \bottomrule
  \end{tabular}
  }
\end{table}

\textbf{Quantitative and Qualitative Superiority.} Empirically, $Z^2$-Sampling stabilizes video generation. On ModelScope, it significantly boosts temporal consistency (CHScore Flow: 70.47 to 72.88; Frame SSIM: 79.13 to 81.03). On CogVideoX, enhanced perceptual continuity is confirmed by a reduction in Frame LPIPS (9.68 to 9.24). 
Qualitatively (Figure~\ref{fig:video}), while baselines struggle with high-frequency flickering (e.g., bamboo leaves) and disjointed rapid motions (e.g., branching lightning), $Z^2$-Sampling locks in structural priors across the video tube. This enforces semantic alignment without accumulating numerical error, synthesizing videos with superior temporal smoothness and uncompromising prompt fidelity.

\subsection{Ablation Studies}

\paragraph{Effectiveness of the Temporal Warmup Phase.}
To validate the necessity of stabilizing early generation trajectories, we ablate the temporal warmup mechanism in $Z^2$-Sampling (Table~\ref{tab:ablation_style2_readable}). While applying $Z^2$-Sampling naively from the initial noise state $t=T$ already yields substantial gains over the standard baseline, it suffers from slight instability. This occurs because the initial denoising steps are highly chaotic, making the dynamically cached Temporal Semantic Surrogate less reliable. Introducing a brief warmup phase allows the continuous diffusion manifold to physically settle before caching the surrogate. As shown in Table~\ref{tab:ablation_style2_readable}, this simple temporal stabilization provides a strict performance boost across all benchmarks. Most notably, it significantly enhances metrics highly sensitive to early structural formation, such as ImageReward on DrawBench (77.38 $\to$ 80.12) and the strict Counting accuracy on GenEval (45.50\% $\to$ 46.80\%), confirming its vital role in preserving fine-grained geometric alignment.

\begin{figure}[t]
  \centering
  \begin{minipage}{0.48\textwidth}
    \centering
    \includegraphics[width=0.95\linewidth]{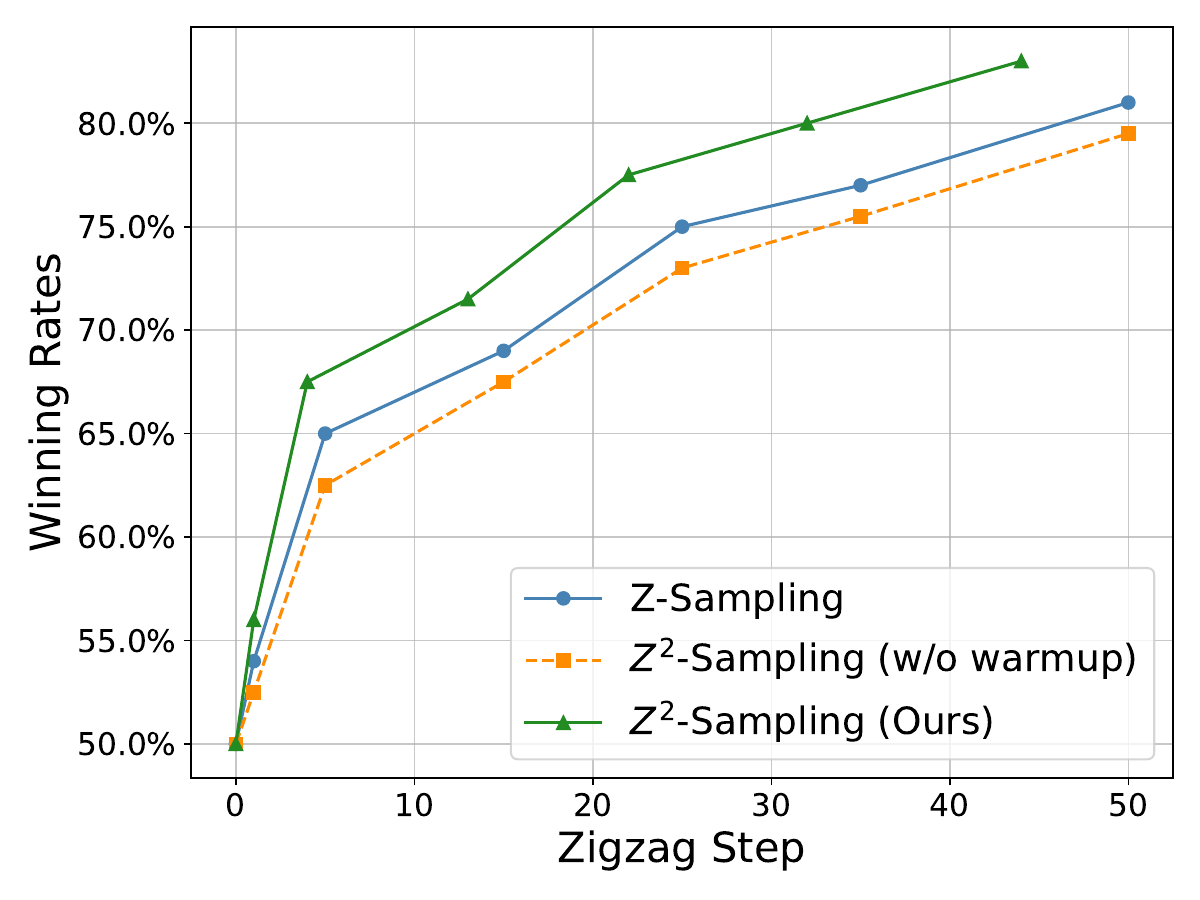}
    \\ \small (a) Winning Rate vs. Zigzag Steps ($\lambda$)
  \end{minipage}\hfill
  \begin{minipage}{0.48\textwidth}
    \centering
    \includegraphics[width=0.95\linewidth]{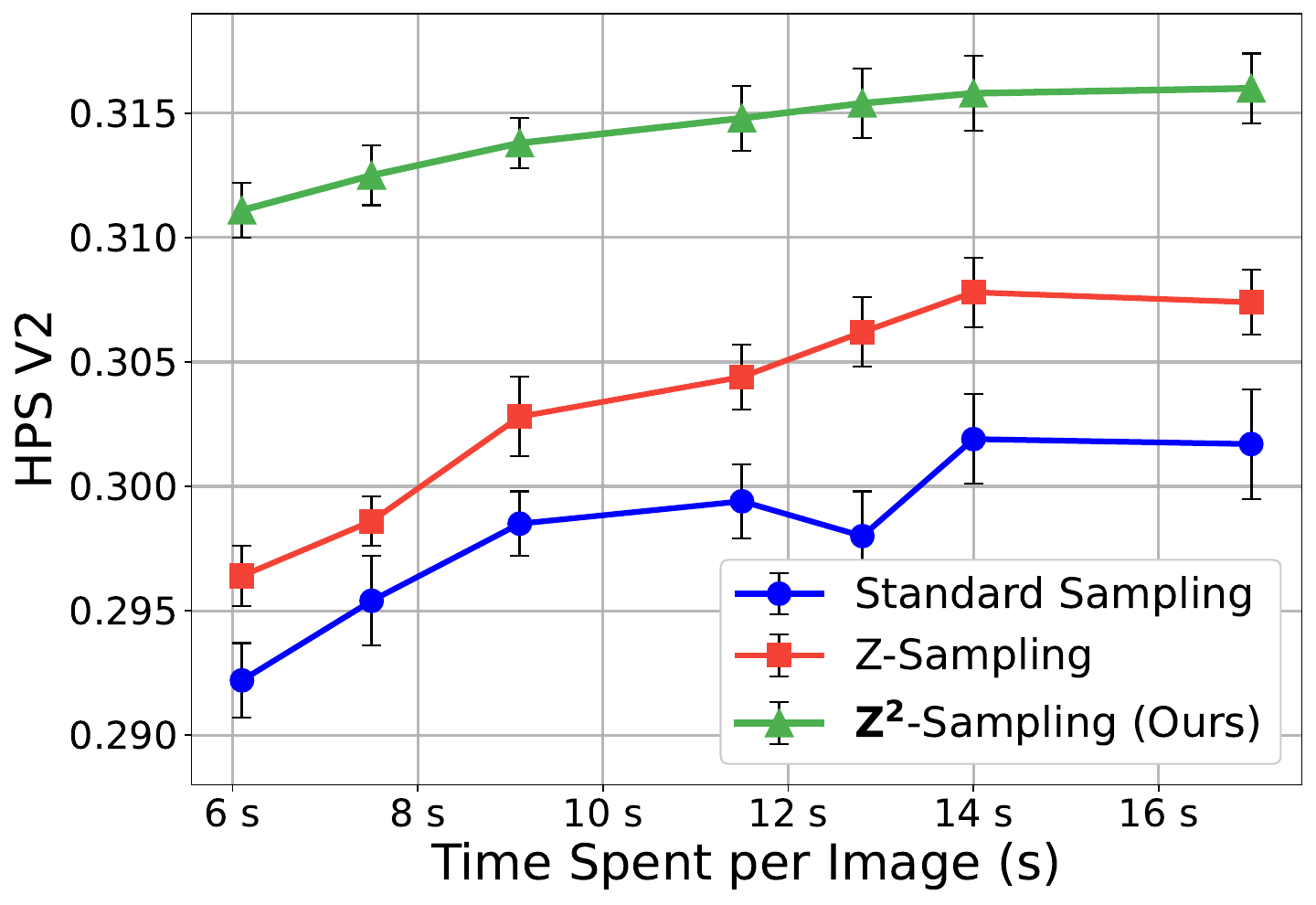}
    \\ \small (b) Semantic Alignment vs. Inference Time
  \end{minipage}
  \caption{\textbf{Shattering the Performance-Efficiency Pareto Frontier.} Both evaluations use HPS v2 on Pick-a-Pic. 
\textbf{(a)} Robustness to the active zigzag diffusion steps ($\lambda$). The vertical axis represents the head-to-head winning rate against the \textbf{Standard Sampling baseline}. Starting from a 50\% tie at $\lambda=0$ (where the process degenerates into standard sampling), the generation quality improves monotonically as $\lambda$ increases, indicating highly effective semantic information gain throughout the path.
  \textbf{(b)} Text-image alignment (HPS v2 score) versus inference time. While explicit Z-Sampling (red) requires severe computational latency to achieve marginal semantic gains over Standard Sampling (blue), $Z^2$-Sampling (green) breaks this barrier. By algebraically eliminating the off-manifold approximation error, it strictly dominates the entire frontier, achieving a higher preference score at its minimum time budget ($\sim$6s) than explicit Z-Sampling does at its peak computational allowance ($\sim$17s).}
  \label{fig:pareto_combined}
\end{figure}

\paragraph{Robustness to Zigzag Steps.}
Figure~\ref{fig:pareto_combined}(a) evaluates generation quality (win rate against standard sampling) versus active zigzag steps ($\lambda$). As $\lambda$ increases from 0 (the baseline), quality improves monotonically across all variants, validating that non-linear probing effectively accumulates semantic information. Notably, $Z^2$-Sampling consistently outperforms explicit Z-Sampling by strictly avoiding the accumulation of off-manifold approximation errors. Furthermore, the performance drop in the non-warmup variant corroborates our theoretical insight: early denoising steps are highly chaotic. A brief warmup phase prevents spatial estimation errors caused by forcing an early temporal surrogate, securing semantic gains.

\begin{table}[H]
  \centering
  \small
  \caption{Ablation study of the warmup phase on SDXL. ``Pick'' denotes PickScore, and ``Avg.'' is the overall score.}
  \label{tab:ablation_style2_readable}
  \begin{tabular*}{\linewidth}{@{\extracolsep{\fill}}lcccccc@{}}
    \toprule
    \multirow{2}{*}{\textbf{Method}} & \multicolumn{2}{c}{\textbf{Pick-a-Pic}} & \multicolumn{2}{c}{\textbf{DrawBench}} & \multicolumn{2}{c}{\textbf{GenEval}} \\
    \cmidrule(lr){2-3} \cmidrule(lr){4-5} \cmidrule(lr){6-7}
    & \textbf{HPS} $\uparrow$ & \textbf{PickScore} $\uparrow$ & \textbf{HPS} $\uparrow$ & \textbf{ImageReward} $\uparrow$ & \textbf{Counting} $\uparrow$ & \textbf{Average} $\uparrow$ \\
    \midrule
    Standard & 29.89 & 21.63 & 28.81 & 60.75 & 33.75 & 52.52 \\
    $Z^2$ (w/o warmup) & \textbf{31.11} & 21.78 & 30.33 & 77.38 & 45.50 & 56.86 \\
    $Z^2$ (w/ warmup) & \textbf{31.11} & \textbf{21.88} & \textbf{30.55} & \textbf{80.12} & \textbf{46.80} & \textbf{57.35} \\
    \bottomrule
  \end{tabular*}
\end{table}

\paragraph{Shattering the Performance-Efficiency Pareto Frontier.} 
Figure~\ref{fig:pareto_combined}(b) illustrates the core empirical achievement of our method by evaluating text-image alignment (HPS v2) against inference time per image. In standard diffusion generation, explicit Z-Sampling (red curve) breaks the standard alignment ceiling but incurs a severe computational penalty, rendering its Pareto performance highly inefficient. Our $Z^2$-Sampling (green curve) structurally shatters this conventional trade-off. By algebraically eliminating the off-manifold approximation error and reducing the computational cost back to the standard 2-NFE baseline via the Temporal Semantic Surrogate, $Z^2$-Sampling strictly dominates the entire frontier. Remarkably, $Z^2$-Sampling achieves a higher human preference score at its minimum time budget ($\sim$6s) than explicit Z-Sampling does at its peak computational allowance ($\sim$17s). This establishes $Z^2$-Sampling as a zero-cost, high-yield semantic alignment paradigm.

\section{Conclusion}
\label{sec:conclusion}

In this paper, we introduced $Z^2$-Sampling, a zero-cost paradigm that resolves the computational and geometric bottlenecks of explicit zigzag trajectories in diffusion models. By leveraging Exact Noise Reuse and ODE solver dualities, our method mathematically bypasses intermediate physical traversals, eliminating off-manifold approximation errors. Furthermore, we integrate a Temporal Semantic Surrogate to restore the standard 2-NFE baseline while synthesizing an implicit curvature penalty to prevent structural degradation under high guidance. Extensive evaluations demonstrate that $Z^2$-Sampling shatters the performance-efficiency Pareto frontier, establishing itself as a  compatible, drop-in replacement for standard CFG across diverse architectures and modalities.

\appendix

\section{Additional Experimental Results}

\begin{table}[htbp]
\centering
\caption{Quantitative results of SEG, CFG++, and our $Z^2$-Sampling (Model: SDXL). "Pick" denotes PickScore.}
\label{tab:seg_cfg}
\begin{tabular}{l ccc c ccc c}
\toprule
\multirow{2}{*}{\textbf{Method}} & \multicolumn{4}{c}{\textbf{Pick-a-Pic}} & \multicolumn{4}{c}{\textbf{DrawBench}} \\
\cmidrule(lr){2-5} \cmidrule(lr){6-9}
& HPS v2$\uparrow$ & AES$\uparrow$ & Pick$\uparrow$ & IR$\uparrow$ & HPS v2$\uparrow$ & AES$\uparrow$ & Pick$\uparrow$ & IR$\uparrow$ \\
\midrule
SEG & 30.53 & 6.12 & 21.42 & 61.57 & 29.60 & 5.66 & 22.15 & 60.42 \\
CFG++ & 30.28 & 6.09 & 21.83 & 67.30 & 28.65 & 5.62 & 22.38 & 62.66 \\
$Z^2$-Sampling (ours) & \textbf{31.11} & \textbf{6.16} & \textbf{21.88} & \textbf{76.63} & \textbf{30.55} & \textbf{5.71} & \textbf{22.39} & \textbf{80.12} \\
\bottomrule
\end{tabular}
\end{table}

\paragraph{Comparison with Training-Free Sampling Baselines.}
To further validate the superiority of our approach, we compare $Z^2$-Sampling against other recent state-of-the-art training-free sampling enhancements, specifically SEG \cite{hong2024smoothedenergyguidanceguiding} and CFG++ \cite{chung2024cfgmanifoldconstrainedclassifierfree}. As reported in Table~\ref{tab:seg_cfg}, $Z^2$-Sampling consistently dominates across all evaluation metrics on both Pick-a-Pic and DrawBench datasets. Notably, it achieves a significant margin in ImageReward (80.12 vs. 62.66 for CFG++ on DrawBench). This comprehensively demonstrates that our proposed algebraic collapse and exact noise reuse mechanism is fundamentally more effective at achieving semantic alignment and human preference than other existing inference-time modifications.

\begin{table}[H]
\centering
\caption{Spatial overhead analysis on SDXL. Crucially, the VRAM overhead introduced by caching the momentum tensor in $Z^2$-Sampling is extremely negligible ($\sim$0.3 MB). By mathematically annihilating the inversion steps, our method gracefully absorbs the architectural memory costs, achieving strict zero-overhead in spatial complexity relative to the Standard DDIM baseline.}
\label{tab:vram_overhead_analysis}
\resizebox{0.9\linewidth}{!}{
    \begin{tabular}{lcc}
    \toprule
    \textbf{Method} & \textbf{Peak VRAM (GB)} $\downarrow$ & \textbf{VRAM Overhead} $\downarrow$ \\
    \midrule
    Standard DDIM & \textbf{7.0137} & \textbf{Base} \\
    Original Z-Sampling & 7.0145 & + 0.82 MB \\
    \textbf{$Z^2$-Sampling (Ours)} & 7.0140 & \textbf{+ 0.31 MB} \\
    \bottomrule
    \end{tabular}
}
\end{table}

\paragraph{Zero Memory Overhead Analysis.}
To rigorously validate the spatial efficiency of our approach, we profile the peak VRAM consumption during SDXL inference. As reported in Table~\ref{tab:vram_overhead_analysis}, the Standard 

\begin{figure}[t]
    \centering
    \includegraphics[width=\linewidth]{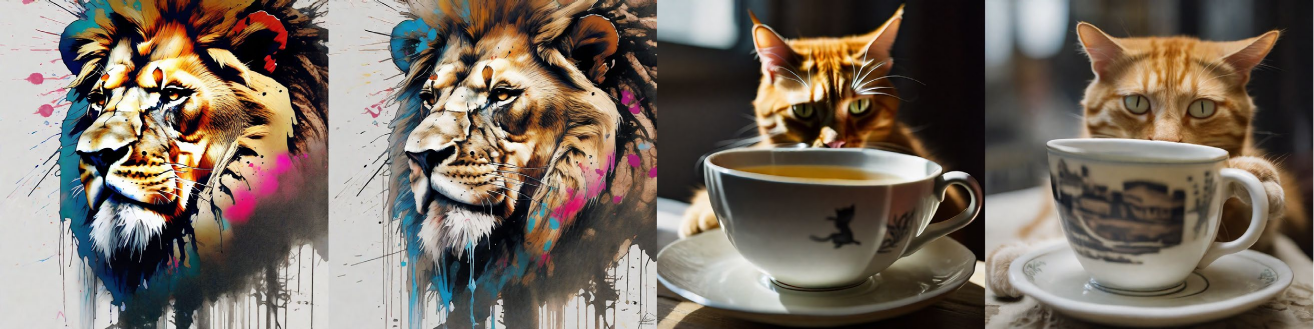}
    \caption{\textbf{Visual validation of the curvature penalty.} At extreme guidance levels, standard CFG suffers from severe oversaturation. Explicit Z-Sampling mitigates this but introduces spatial drift artifacts. $Z^2$-Sampling strictly internalizes the curvature penalty ($-h \gamma_1 \nabla_{\Delta \bm{v}} \bm{v}$) without off-manifold error, successfully suppressing the oversaturation caused by implicit double guidance ($2\gamma_1$) to produce high-fidelity results.}
    \label{fig:high_cfg_visualization}
\end{figure}

DDIM establishes a foundational memory footprint of 7.0137 GB. Explicit Z-Sampling increases this footprint to 7.0145 GB (a +0.82 MB overhead), primarily due to the instantiation of a dual-scheduler system and the accumulation of intermediate latent states during its multi-pass forward-backward operations. 

In stark contrast, our proposed $Z^2$-Sampling restricts the peak VRAM to 7.0140 GB. The empirical overhead introduced by dynamically caching the Temporal Semantic Surrogate (the momentum tensor) is merely 0.31 MB. This microscopic difference confirms our theoretical intuition: by mathematically annihilating the intermediate physical traversals and strictly bypassing the inversion scheduler, $Z^2$-Sampling seamlessly absorbs the memory cost of the surrogate cache. Consequently, our method achieves strict zero-overhead in spatial complexity, confirming it as a highly practical, drop-in replacement that places no additional burden on consumer hardware.

\begin{figure}[H]
    \includegraphics[width=\linewidth]{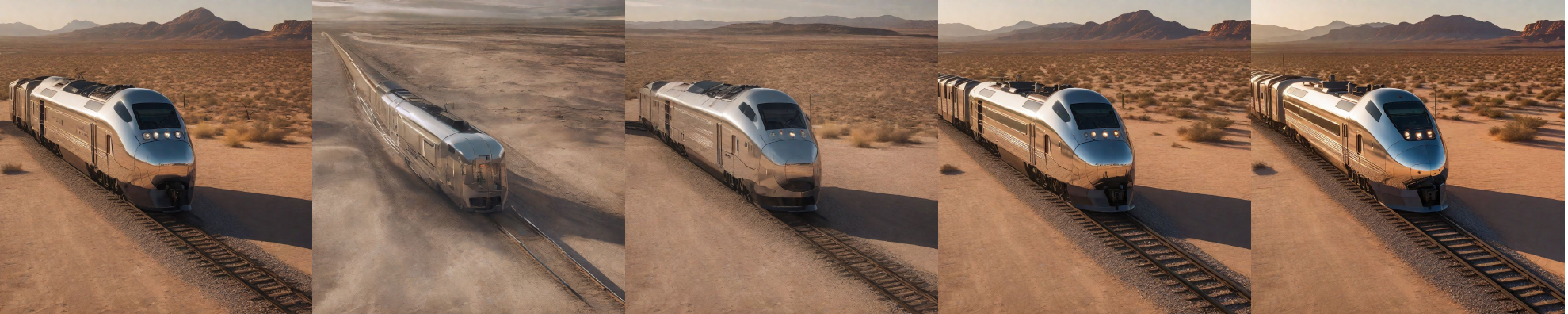}
    \caption{\textbf{Qualitative comparison across varying guidance scales ($\gamma_1$).} The leftmost column displays the Standard Sampling baseline. The subsequent columns, from left to right, illustrate the results of our proposed $Z^2$-Sampling with $\gamma_1 \in \{1, 2, 3, 4,\}$. As $\gamma_1$ increases, $Z^2$-Sampling progressively deepens semantic alignment and visual richness. Crucially, by algebraically eliminating off-manifold spatial drift, it successfully suppresses the color oversaturation and structural degradation typically associated with elevated guidance scales.}
    \label{fig:gamma1_vis}
\end{figure}

\section{Visualizing the Curvature Penalty at Ultra-High Guidance}
\label{sec:appendix_high_cfg}

Per Theorem \ref{thm:modified_equation}, the Effective Vector Field of $Z^2$-Sampling couples implicitly doubled guidance ($2\gamma_1 \Delta \bm{v}$) with a curvature penalty ($-h \gamma_1 \nabla_{\Delta \bm{v}} \bm{v}$). While a naive $2\gamma_1$ scale causes catastrophic oversaturation, this penalty acts as a Riemannian centripetal force neutralizing it. Figure \ref{fig:high_cfg_visualization} validates this.At extreme guidance (e.g., $\gamma_1 \ge 11$), standard CFG suffers structural collapse and oversaturation from blind linear extrapolation. Explicit Z-Sampling partially resists this via its equivalent penalty, but its off-manifold traversal introduces spatial drift $\tau(t)$ (Appendix \ref{app:tau_annihilation}), causing color degradation and localized artifacts.By algebraically enforcing Exact Noise Reuse, $Z^2$-Sampling eradicates spatial drift $\tau(t)$ while preserving the curvature penalty. Consequently, it suppresses oversaturation triggered by implicit double guidance, yielding structurally coherent, natural, and highly text-aligned images at extreme scales.

\section{Related Works}

\subsection{Semantic Information in Latent Space}
Recent works have demonstrated that the prior information embedded within the initial noise latent significantly impacts the quality and alignment of image generation \cite{samuel2023generatingimagesrareconcepts, xu2025goodseedmakesgood}. For instance, observations indicate that specific regions within random latents can trigger the generation of localized concepts \cite{mao2024lotterytickethypothesisdenoising}, while other approaches propose Golden Noise \cite{zhou2025goldennoisediffusionmodels} to leverage a semantic accumulation approach. Furthermore, studies have highlighted that even minor perturbations in the initial latent can lead to drastic shifts in the synthesized outputs of diffusion models \cite{poyuan2023syntheticshiftsinitialseed}. Consequently, explicitly injecting semantic information—such as low-frequency wavelengths—into Gaussian noise has proven effective in enhancing image quality and prompt alignment \cite{wu2023human}. Building upon these insights, our work investigates semantic information from a guidance perspective, implicitly integrating it into the sampling trajectory without the need for explicit reference data.

\subsection{Sampling Strategies of Diffusion Models}
To refine the diffusion sampling process, various strategies have been proposed. Early methods introduced Resampling \cite{lugmayr2022repaintinpaintingusingdenoising}, which involves injecting random noise and performing iterative back-and-forth steps at each timestep—a paradigm later adopted for video generation. Similarly, methods like IRFDS \cite{yang2025texttoimagerectifiedflowplugandplay} employ a pretrained rectified flow model to provide a prior, optimizing the initial latent for image editing tasks. However, they overlook the structural importance of the inverted latent by simply applying random noise, thereby limiting prompt adherence. In the realm of structural consistency, approaches such as Tune-a-Video \cite{wu2023tuneavideooneshottuningimage} incorporate a denoising-inversion paradigm; yet, their end-to-end approach remains suboptimal due to the neglect of the guidance gap.
Other approaches address spatial inconsistencies by proposing adaptive guidance mechanisms based on semantic segmentation \cite{shen2024rethinkingspatialinconsistencyclassifierfree}, but their reliance on attention-level interventions restricts adaptability and robustness. Constraint-based methods attempt to refine generation from a manifold perspective—such as substituting conditional noise with unconditional noise \cite{chung2024cfgmanifoldconstrainedclassifierfree} or applying spherical Gaussian constraints—though these often yield marginal improvements or require external reference data.

Notably, \textbf{Z-Sampling} \cite{bai2024zigzagdiffusionsamplingdiffusion} has emerged as a powerful technique to capture data manifold curvature by explicitly traversing multi-step forward-backward trajectories in the latent space. While it significantly improves semantic alignment and transfers effectively to various generative paradigms, its explicit traversals introduce two critical bottlenecks: they triple the Neural Function Evaluation (NFE) cost per step and force network evaluations at intermediate, off-manifold states. This positional mismatch injects an unconstrained spatial approximation error ($\tau(t)$) that degrades generation quality. To resolve this fundamental flaw, we propose \textbf{$Z^2$-Sampling}, a zero-cost paradigm that algebraically collapses the intermediate physical traversals. By coupling exact noise reuse with a dynamically cached Temporal Semantic Surrogate, $Z^2$-Sampling strictly annihilates the off-manifold spatial drift ($\tau(t) \equiv 0$) and restores the standard 2-NFE computational baseline. Consequently, it structurally shatters the performance-efficiency Pareto frontier, offering a robust, highly consistent semantic sampling trajectory without any computational overhead.

\section{Extended Theoretical Analysis and Proofs}
\label{sec:appendix_theory}

In this section, we provide the rigorous mathematical derivations omitted in the main text due to space constraints. Specifically, we formally verify the algebraic dualities across various deterministic ODE solvers, provide the detailed step-by-step derivation of the Effective Vector Field via Backward Error Analysis (BEA), and mathematically formalize the spatial drift $\tau(t)$ to prove its strict annihilation in $Z^2$-Sampling.

\subsection{Proof of Lemma 1}
\label{app:solver_dualities}

Lemma \ref{lemma:involution} establishes the universal topological constraints $A_t^{-1} B_t \equiv -C_t$ and $A_t C_t \equiv -B_t$ for deterministic affine diffusion solvers. We explicitly verify this for the three most widely used solvers: Variance Preserving (DDIM), Flow Matching (Euler), and Spherical formulations.

\paragraph{1. Standard DDIM ($\bm{\epsilon}$-prediction)}
For the DDIM formulation \cite{song2022denoisingdiffusionimplicitmodels}, the forward integration step $\Phi^t(\bm{x}_t; \bm{\epsilon})$ from timestep $t$ to $t-1$ is defined by:
\begin{equation}
    A_t = \sqrt{\frac{\alpha_{t-1}}{\alpha_t}}, \quad B_t = \sqrt{1-\alpha_{t-1}} - \sqrt{\frac{\alpha_{t-1}}{\alpha_t}} \sqrt{1-\alpha_t}
\end{equation}
The exact deterministic inversion step $\Psi^t(\bm{x}_{t-1}; \bm{\epsilon})$ from $t-1$ back to $t$ swaps the time indices:
\begin{equation}
    A_t^{-1} = \sqrt{\frac{\alpha_t}{\alpha_{t-1}}}, \quad C_t = \sqrt{1-\alpha_t} - \sqrt{\frac{\alpha_t}{\alpha_{t-1}}} \sqrt{1-\alpha_{t-1}}
\end{equation}
To strictly verify the duality $A_t^{-1} B_t \equiv -C_t$, we expand the product:
\begin{align}
    A_t^{-1} B_t &= \sqrt{\frac{\alpha_t}{\alpha_{t-1}}} \left( \sqrt{1-\alpha_{t-1}} - \sqrt{\frac{\alpha_{t-1}}{\alpha_t}} \sqrt{1-\alpha_t} \right) \nonumber \\
    &= \sqrt{\frac{\alpha_t}{\alpha_{t-1}}} \sqrt{1-\alpha_{t-1}} - \sqrt{1-\alpha_t} \nonumber \\
    &= -\left( \sqrt{1-\alpha_t} - \sqrt{\frac{\alpha_t}{\alpha_{t-1}}} \sqrt{1-\alpha_{t-1}} \right) \nonumber \\
    &\equiv -C_t
\end{align}
Multiplying both sides by $A_t$ inherently yields $B_t = -A_t C_t$, confirming Lemma 1.

\paragraph{2. Flow Matching / Euler Solver ($\bm{v}$-prediction)}
For continuous-time Flow Matching \cite{lipman2023flowmatchinggenerativemodeling} and Rectified Flows, the trajectory is integrated via the explicit Euler method. The forward step is simply $\bm{x}_{t-1} = \bm{x}_t + (\sigma_{t-1} - \sigma_t) \bm{v}_\theta$:
\begin{equation}
    A_t = 1, \quad B_t = \sigma_{t-1} - \sigma_t
\end{equation}
The exact inversion step reverses the integration direction:
\begin{equation}
    A_t^{-1} = 1, \quad C_t = \sigma_t - \sigma_{t-1}
\end{equation}
It trivially follows that $A_t^{-1} B_t = 1 \cdot (\sigma_{t-1} - \sigma_t) = -(\sigma_t - \sigma_{t-1}) \equiv -C_t$.

\paragraph{3. Spherical Interpolation ($\bm{v}$-prediction)}
For variance-preserving schedules framed in spherical coordinates, where the angle is $\theta_t = \arccos(\sqrt{\alpha_t})$, the transitions are given by trigonometric functions:
\begin{equation}
    A_t = \cos(\Delta\theta), \quad B_t = \sin(\Delta\theta)
\end{equation}
where $\Delta\theta = \theta_{t-1} - \theta_t$. The exact inversion coefficients are defined by reversing the angular step:
\begin{equation}
    A_t^{-1} = \frac{1}{\cos(\Delta\theta)}, \quad C_t = - \frac{\sin(\Delta\theta)}{\cos(\Delta\theta)}
\end{equation}
Evaluating the product $A_t C_t$:
\begin{equation}
    A_t C_t = \cos(\Delta\theta) \left( - \frac{\sin(\Delta\theta)}{\cos(\Delta\theta)} \right) = -\sin(\Delta\theta) \equiv -B_t
\end{equation}
This confirms that our algebraic collapse uniformly generalizes to spherical solver geometries.

\subsection{Proof of Theorem 4}
\label{app:bea_derivation}

We present the full mathematical derivation for Theorem \ref{thm:modified_equation}. Applying an explicit Euler discretization maps the coefficients to $A_t = 1, B_t = -h, C_t = h$, where $h$ is the integration step size.

The single algebraically collapsed intermediate state in $Z^2$-Sampling is:
\begin{equation}
    \tilde{\bm{x}}_t = \bm{x}_t - h \gamma_1 \Delta \bm{v}(\bm{x}_t)
\end{equation}
In the subsequent step, the network is evaluated at this shifted state: $\tilde{\bm{v}}_t \triangleq \bm{v}(\tilde{\bm{x}}_t, \gamma_1)$. We perform a multivariable Taylor expansion of $\bm{v}(\cdot, \gamma_1)$ centered around the pristine anchor $\bm{x}_t$:
\begin{align}
    \tilde{\bm{v}}_t &= \bm{v}\big(\bm{x}_t - h \gamma_1 \Delta \bm{v}(\bm{x}_t), \gamma_1\big) \nonumber \\
    &= \bm{v}(\bm{x}_t, \gamma_1) + \mathbf{J}_{\bm{v}}(\bm{x}_t, \gamma_1) \big( -h \gamma_1 \Delta \bm{v}(\bm{x}_t) \big) + \mathcal{O}(h^2) \nonumber \\
    &= \bm{v}(\bm{x}_t, \gamma_1) - h \gamma_1 \mathbf{J}_{\bm{v}}(\bm{x}_t, \gamma_1) \Delta \bm{v}(\bm{x}_t) + \mathcal{O}(h^2)
\end{align}
Using the formal definition of the directional derivative, \\$\mathbf{J}_{\bm{v}}(\bm{x}_t, \gamma_1) \Delta \bm{v}(\bm{x}_t) \triangleq \nabla_{\Delta \bm{v}} \bm{v}(\bm{x}_t, \gamma_1)$, we simplify this to:
\begin{equation}
    \tilde{\bm{v}}_t = \bm{v}(\bm{x}_t, \gamma_1) - h \gamma_1 \nabla_{\Delta \bm{v}} \bm{v}(\bm{x}_t, \gamma_1) + \mathcal{O}(h^2)
\end{equation}
The final integration step translates $\bm{x}_t$ to the updated state $\bm{x}_{t-1}^{\mathrm{new}}$ using the aggregated velocity:
\begin{align*}
    \bm{x}_{t-1}^{\mathrm{new}} &= \bm{x}_t - h \Big( \tilde{\bm{v}}_t + \gamma_1 \Delta \bm{v}(\bm{x}_t) \Big) \nonumber \\
    &= \bm{x}_t - h \Big[ \bm{v}(\bm{x}_t, \gamma_1) - h \gamma_1 \nabla_{\Delta \bm{v}} \bm{v}(\bm{x}_t, \gamma_1) + \gamma_1 \Delta \bm{v}(\bm{x}_t) \Big] + \mathcal{O}(h^3)
\end{align*}
Because standard CFG velocity is defined as $\bm{v}(\bm{x}_t, \gamma_1) = \bm{v}_{uc}(\bm{x}_t) + \gamma_1 \Delta \bm{v}(\bm{x}_t)$, we substitute this into the bracket:
\begin{equation}
    \bm{x}_{t-1}^{\mathrm{new}} = \bm{x}_t - h \Big[ \bm{v}_{uc}(\bm{x}_t) + 2\gamma_1 \Delta \bm{v}(\bm{x}_t) - h \gamma_1 \nabla_{\Delta \bm{v}} \bm{v}(\bm{x}_t, \gamma_1) \Big] + \mathcal{O}(h^3)
\end{equation}
Dividing by $h$ to express this discrete transition as a continuous-time differential equation yields the Effective Vector Field $\bm{v}_{\mathrm{eff}}$:
\begin{equation*}
    \bm{v}_{\mathrm{eff}}(\bm{x}) = \frac{\bm{x}_t - \bm{x}_{t-1}^{\mathrm{new}}}{h} = \bm{v}_{uc}(\bm{x}) + 2\gamma_1 \Delta \bm{v}(\bm{x}) - h \gamma_1 \nabla_{\Delta \bm{v}} \bm{v}(\bm{x}, \gamma_1) + \mathcal{O}(h^2)
\end{equation*}
By inherently synthesizing the penalty $\mathcal{R}(\bm{x}) = -h \gamma_1 \nabla_{\Delta \bm{v}} \bm{v}(\bm{x}, \gamma_1)$, $Z^2$-Sampling projects the generative trajectory against the direction of steepest density descent, acting as a \textit{Riemannian centripetal force} that dynamically bends the trajectory back onto the data manifold.

\subsection{Annihilation of Spatial Drift}
\label{app:tau_annihilation}

As emphasized in Remark 2, explicit Z-Sampling suffers from severe spatial drift. Here, we formally define this error and demonstrate how Exact Noise Reuse algebraically annihilates it.

\paragraph{The Emergence of $\tau(t)$ in Explicit Z-Sampling}
In the explicitly traversed zigzag path, the network moves to the forward state $\bm{x}_{t-1}^{\mathrm{exp}} = \bm{x}_t - h \bm{v}(\bm{x}_t, \gamma_1)$. Because the CFG scale $\gamma_1$ is typically large, this explicitly constructed point falls strictly outside the high-density local manifold tube $\mathcal{T}_\delta(\mathcal{M})$. 
The physical inversion evaluates the unconditional field at this uncorrected off-manifold point:
\begin{equation}
    \tilde{\bm{x}}_t^{\mathrm{exp}} = \bm{x}_{t-1}^{\mathrm{exp}} + h \bm{v}_{uc}(\bm{x}_{t-1}^{\mathrm{exp}})
\end{equation}
If we Taylor expand $\bm{v}_{uc}(\bm{x}_{t-1}^{\mathrm{exp}})$ around the on-manifold anchor $\bm{x}_t$:
\begin{equation}
    \bm{v}_{uc}(\bm{x}_{t-1}^{\mathrm{exp}}) = \bm{v}_{uc}(\bm{x}_t) - h \mathbf{J}_{\bm{v}_{uc}}(\bm{x}_t) \bm{v}(\bm{x}_t, \gamma_1) + \mathcal{O}(h^2)
\end{equation}
Substituting this back, we find:
\begin{align}
    \tilde{\bm{x}}_t^{\mathrm{exp}} &= \bm{x}_t - h \bm{v}(\bm{x}_t, \gamma_1) + h \Big[ \bm{v}_{uc}(\bm{x}_t) - h \mathbf{J}_{\bm{v}_{uc}}(\bm{x}_t) \bm{v}(\bm{x}_t, \gamma_1) \Big] \nonumber \\
    &= \bm{x}_t - h \gamma_1 \Delta \bm{v}(\bm{x}_t) \underbrace{- h^2 \mathbf{J}_{\bm{v}_{uc}}(\bm{x}_t) \bm{v}(\bm{x}_t, \gamma_1)}_{\text{Spatial Approximation Error } \tau(t)} + \mathcal{O}(h^3)
\end{align}
The explicit traversal injects a continuous-time spatial drift $\tau(t) \approx - h^2 \mathbf{J}_{\bm{v}_{uc}}(\bm{x}_t) \bm{v}(\bm{x}_t, \gamma_1)$. Critically, because $\bm{x}_{t-1}^{\mathrm{exp}}$ is off-manifold, the Lipschitz constant bounds fail (as per Assumption 1), and the Jacobian norm $\| \mathbf{J}_{\bm{v}_{uc}} \|$ becomes unconstrained. This introduces an irreversible numerical drift, manifesting visually as structural degradation and color oversaturation.

\paragraph{The Strict Annihilation ($\tau(t) \equiv 0$) in $Z^2$-Sampling}
By enforcing Exact Noise Reuse, $Z^2$-Sampling rigidly anchors the inversion step exactly at the on-manifold coordinate $\bm{x}_t$, reusing the pristine anchor $\bm{v}_{uc}(\bm{x}_t)$. Consequently, the Taylor expansion entirely collapses:
\begin{equation}
    \tilde{\bm{x}}_t = \bm{x}_t - h \gamma_1 \Delta \bm{v}(\bm{x}_t)
\end{equation}
Comparing this to the explicit case mathematically guarantees that $\tau(t)$ is physically bypassed and strictly evaluates to zero ($\tau(t) \equiv 0$). This ensures the complete elimination of off-manifold approximation errors while retaining the continuous-time curvature penalty.

\begin{figure}[H]
    \centering
    \includegraphics[width=\textwidth]{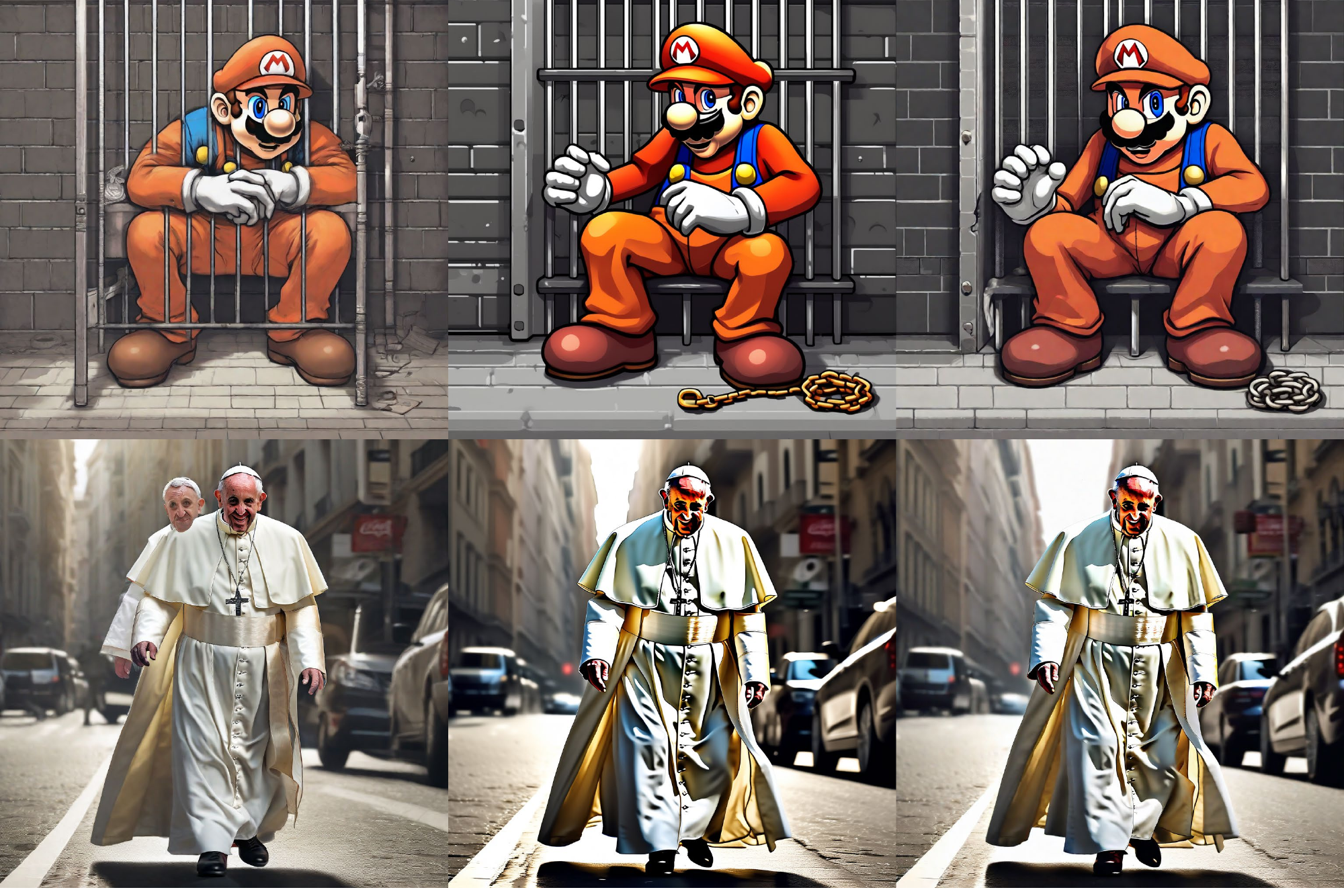}
    \caption{Failure cases of the proposed method.}
    \label{fig:failure_cases}
\end{figure}

\section{Failure Case Analysis}
\label{sec:appendix_failure}

While $Z^2$-Sampling structurally eliminates the off-manifold spatial drift $\tau(t)$ and significantly mitigates the color oversaturation artifacts inherent in explicit Z-Sampling, it is not entirely immune to the stochastic challenges of diffusion manifolds. As illustrated in Figure~\ref{fig:failure_cases}, we identify two primary failure modes that warrant further investigation.

\paragraph{Surrogate-Induced Structural Inconsistency.} 
The first row of Figure~\ref{fig:failure_cases} demonstrates that although $Z^2$-Sampling produces a more natural color palette and sharper textures than explicit Z-Sampling, it may occasionally inherit or amplify structural errors. For instance, in the Mario case, $Z^2$-Sampling results in an anatomical inconsistency (six fingers). This suggests that while the Temporal Semantic Surrogate error is theoretically bounded at $\mathcal{O}(h^2)$, a rapid shift in the data manifold geometry during the zigzag phase can cause the cached $\Delta \bm{\epsilon}$ to misguide fine-grained topological features.

\paragraph{Saturation Overflow in High-Density Priors.}
The second row illustrates a "saturation overflow" effect. In scenarios where the Standard Sampling baseline already achieves high semantic density and complex lighting (e.g., the Pope case), any further semantic enhancement—whether via Z or $Z^2$-Sampling—inevitably pushes the latent state into the extreme extrapolation regime of the CFG field. In these instances, the enhanced semantic guidance becomes 
redundant, leading to oversaturation and loss of dynamic range, regardless of our method's ability to eliminate spatial drift. This indicates that $Z^2$-Sampling is most effective as a corrective mechanism for \textit{under-aligned} prompts rather than a universal intensity booster.

\section{Detailed Experimental Configurations}
\label{app:exp_details}

To facilitate reproducibility, we summarize the specific hardware environment and the hyper-parameters used for text-to-video generation in Table~\ref{tab:video_params_appendix}. All experiments were executed using a single localized workstation.

\section{Implementation Details}
\label{app:implementation}

In this section, we provide the specific hardware configurations and hyper-parameters utilized for our video generation experiments. To ensure high-quality synthesis within a constrained VRAM budget, we employ the settings detailed in Table~\ref{tab:video_params_appendix}.

\begin{table}[H]
\centering
\small
\caption{Hardware environment and video generation configurations for $Z^2$-Sampling.}
\label{tab:video_params_appendix}
\begin{tabular}{@{}ll@{}}
\toprule
\textbf{Configuration Item} & \textbf{Value / Detail} \\ \midrule
\rowcolor[HTML]{F9F9F9} 
\multicolumn{2}{l}{\textit{Hardware \& Reproducibility}} \\
GPU Device & NVIDIA GeForce RTX 4090 (24GB VRAM) \\
Global Random Seed & 42 \\
Precision & FP16 (Mixed Precision) \\ \midrule
\rowcolor[HTML]{F9F9F9} 
\multicolumn{2}{l}{\textit{Video Generation}} \\
Resolution ($H \times W$) & $480 \times 720$ \\
Total Frames & 16 \\
Output FPS & 8 \\
Denoising Steps ($T$) & 50 \\
Guidance Scale ($\gamma_1$) & 6.0 \\
Warmup Duration ($W$) & 10 steps \\ \bottomrule
\end{tabular}
\end{table}

\section{Overview of Evaluated Models, Datasets, and Metrics}\label{sec:appendix_overview}In this section, we provide a comprehensive overview of the diffusion models, datasets, and evaluation metrics utilized throughout our empirical analysis to ensure self-containment and clarity.\subsection{Diffusion Models and Architectures}Our evaluations span a diverse set of fundamental architectures, including both standard U-Nets and recent Diffusion Transformers.
\begin{itemize}\item \textbf{Stable Diffusion 2.1 (SD-2.1)} \cite{rombach2022highresolutionimagesynthesislatent}: A foundational open-source latent diffusion model based on the U-Net architecture, trained on filtered subsets of the LAION dataset for high-resolution text-to-image synthesis.\item \textbf{Stable Diffusion XL (SDXL)} \cite{podell2023sdxlimprovinglatentdiffusion}: A significantly scaled-up successor to SD-2.1, featuring a massive U-Net backbone, a dual text-encoder system, and enhanced micro-conditioning mechanisms, allowing for superior compositional capability and photorealism.\item \textbf{DreamShaper-xl-v2-turbo} \cite{sauer2023adversarialdiffusiondistillation}: A highly optimized, distilled variant of SDXL (utilizing Adversarial Diffusion Distillation) designed to generate high-fidelity images in exceptionally few denoising steps (e.g., 4 steps), serving as our benchmark for ultra-fast sampling regimes.\item \textbf{Hunyuan-DiT} \cite{li2024hunyuandit}: A state-of-the-art text-to-image model that replaces the traditional convolutional U-Net with a scalable Diffusion Transformer (DiT) architecture, demonstrating strong semantic adherence and Chinese-English bilingual comprehension.\item \textbf{CogVideoX-2B} \cite{yang2024cogvideox}: A 2-billion parameter text-to-video generation model based on the DiT architecture. It excels in maintaining temporal coherence and understanding complex semantic motion prompts over continuous frames.\item \textbf{ModelScope-1.7B} \cite{wang2023modelscope}: A robust baseline for text-to-video synthesis utilizing a multi-stage spatiotemporal U-Net architecture to progressively generate and refine video sequences.\item \textbf{Align Your Steps (AYS)} \cite{sabour2024alignstepsoptimizingsampling}: A training-free sampling optimization framework that analytically determines the optimal timestep schedule for inference, which we combined with $Z^2$-Sampling to demonstrate orthogonal compatibility.\item \textbf{Diffusion-DPO} \cite{wallace2023diffusionmodelalignmentusing}: A training-based alignment paradigm (Direct Preference Optimization applied to diffusion) that fine-tunes model weights based on human preference datasets.\end{itemize}\subsection{Evaluation Datasets and Benchmarks}We measure semantic alignment and generation quality across several standardized, highly challenging benchmarks.\begin{itemize}\item \textbf{Pick-a-Pic} \cite{kirstain2023pickapicopendatasetuser}: A large-scale, open-source dataset consisting of complex text prompts and crowdsourced human preference labels, heavily used to evaluate text-to-image alignment and overall visual appeal.\item \textbf{DrawBench} \cite{saharia2022photorealistictexttoimagediffusionmodels}: Introduced alongside Imagen, this benchmark comprises a curated set of challenging prompts designed specifically to probe the weaknesses of text-to-image models, including dense compositions, text rendering, and unusual spatial relations.\item \textbf{GenEval} \cite{ghosh2023genevalobjectfocusedframeworkevaluating}: An object-focused evaluation framework that rigorously tests fine-grained compositional capabilities. It isolates specific skills such as single/multiple object rendering, precise counting, spatial positioning, and color attribution.\item \textbf{ChronoMagic-Bench-150} \cite{yuan2024chronomagic}: A specialized text-to-video benchmark focused on evaluating both prompt fidelity and the physical/temporal consistency of generated video sequences across diverse motion categories.\end{itemize}\subsection{Evaluation Metrics}To provide a holistic assessment, we employ a suite of automated metrics targeting human preference, semantic alignment, and temporal consistency.\paragraph{Text-to-Image Metrics:}\begin{itemize}\item \textbf{Human Preference Score v2 (HPS v2)} \cite{wu2023human}: A scoring model trained on a vast dataset of human choices to predict human preference regarding text-image alignment and aesthetic quality.\item \textbf{PickScore} \cite{kirstain2023pickapicopendatasetuser}: A specialized metric trained directly on the Pick-a-Pic dataset, functioning as a highly accurate proxy for human visual preference in text-to-image generation.\item \textbf{ImageReward (IR)} \cite{xu2023imagerewardlearningevaluatinghuman}: A comprehensive reward model trained to align with human aesthetic values and strict text-prompt adherence, capturing nuances often missed by standard CLIP scores.\item \textbf{Aesthetic Score (AES)} \cite{schuhmann2022laion5bopenlargescaledataset}: A neural evaluator trained to predict the inherent visual appeal and artistic quality of an image, independent of strict text alignment.\end{itemize}\paragraph{Text-to-Video Metrics:}\begin{itemize}\item \textbf{CLIP SIM (CLIP Similarity)} \cite{radford2021learningtransferablevisualmodels}: Measures the cosine similarity between the text prompt embedding and the visual embeddings of generated frames, quantifying semantic text-to-video alignment.\item \textbf{CHScore Flow} \cite{yuan2024chronomagic}: Evaluates the temporal consistency and logical motion progression of video sequences by analyzing optical flow dynamics across consecutive frames.\item \textbf{Frame LPIPS} \cite{zhang2018perceptual}: Evaluates the temporal smoothness by computing the Learned Perceptual Image Patch Similarity between consecutive frames. Lower values indicate fewer abrupt perceptual shifts.\item \textbf{Frame SSIM} \cite{nilsson2020understandingssim}: Computes the Structural Similarity Index Measure across adjacent frames. Higher values denote better preservation of structural elements over time.\end{itemize}

\section{Limitations}
\label{sec:limitations}

Despite the significant performance-efficiency gains achieved by $Z^2$-Sampling, this work has certain limitations that open avenues for future research:

\paragraph{Empirical Dependency of the Warmup Phase.}
$Z^2$-Sampling currently requires a brief warmup phase ($W$) to stabilize the Temporal Semantic Surrogate. While empirical results suggest $W=5$ for images and $W=10$ for videos are generally optimal, these values are heuristically determined. The ideal duration may vary with specific noise schedules, manifold complexity, or total sampling steps. Thus, an adaptive mechanism to dynamically transition from standard CFG to $Z^2$-Sampling based on trajectory curvature or signal-to-noise ratios remains unexplored.

\bibliography{software}

@misc{song2021scorebasedgenerativemodelingstochastic,
      title={Score-Based Generative Modeling through Stochastic Differential Equations}, 
      author={Yang Song and Jascha Sohl-Dickstein and Diederik P. Kingma and Abhishek Kumar and Stefano Ermon and Ben Poole},
      year={2021},
      eprint={2011.13456},
      archivePrefix={arXiv},
      primaryClass={cs.LG},
      url={https://arxiv.org/abs/2011.13456}, 
}

@misc{ho2020denoisingdiffusionprobabilisticmodels,
      title={Denoising Diffusion Probabilistic Models}, 
      author={Jonathan Ho and Ajay Jain and Pieter Abbeel},
      year={2020},
      eprint={2006.11239},
      archivePrefix={arXiv},
      primaryClass={cs.LG},
      url={https://arxiv.org/abs/2006.11239}, 
}

@misc{song2022denoisingdiffusionimplicitmodels,
      title={Denoising Diffusion Implicit Models}, 
      author={Jiaming Song and Chenlin Meng and Stefano Ermon},
      year={2022},
      eprint={2010.02502},
      archivePrefix={arXiv},
      primaryClass={cs.LG},
      url={https://arxiv.org/abs/2010.02502}, 
}

@misc{rombach2022highresolutionimagesynthesislatent,
      title={High-Resolution Image Synthesis with Latent Diffusion Models}, 
      author={Robin Rombach and Andreas Blattmann and Dominik Lorenz and Patrick Esser and Björn Ommer},
      year={2022},
      eprint={2112.10752},
      archivePrefix={arXiv},
      primaryClass={cs.CV},
      url={https://arxiv.org/abs/2112.10752}, 
}

@misc{podell2023sdxlimprovinglatentdiffusion,
      title={SDXL: Improving Latent Diffusion Models for High-Resolution Image Synthesis}, 
      author={Dustin Podell and Zion English and Kyle Lacey and Andreas Blattmann and Tim Dockhorn and Jonas Müller and Joe Penna and Robin Rombach},
      year={2023},
      eprint={2307.01952},
      archivePrefix={arXiv},
      primaryClass={cs.CV},
      url={https://arxiv.org/abs/2307.01952}, 
}

@misc{lipman2023flowmatchinggenerativemodeling,
      title={Flow Matching for Generative Modeling}, 
      author={Yaron Lipman and Ricky T. Q. Chen and Heli Ben-Hamu and Maximilian Nickel and Matt Le},
      year={2023},
      eprint={2210.02747},
      archivePrefix={arXiv},
      primaryClass={cs.LG},
      url={https://arxiv.org/abs/2210.02747}, 
}

@misc{liu2022flowstraightfastlearning,
      title={Flow Straight and Fast: Learning to Generate and Transfer Data with Rectified Flow}, 
      author={Xingchao Liu and Chengyue Gong and Qiang Liu},
      year={2022},
      eprint={2209.03003},
      archivePrefix={arXiv},
      primaryClass={cs.LG},
      url={https://arxiv.org/abs/2209.03003}, 
}

@misc{ho2022videodiffusionmodels,
      title={Video Diffusion Models}, 
      author={Jonathan Ho and Tim Salimans and Alexey Gritsenko and William Chan and Mohammad Norouzi and David J. Fleet},
      year={2022},
      eprint={2204.03458},
      archivePrefix={arXiv},
      primaryClass={cs.CV},
      url={https://arxiv.org/abs/2204.03458}, 
}

@misc{blattmann2023stablevideodiffusionscaling,
      title={Stable Video Diffusion: Scaling Latent Video Diffusion Models to Large Datasets}, 
      author={Andreas Blattmann and Tim Dockhorn and Sumith Kulal and Daniel Mendelevitch and Maciej Kilian and Dominik Lorenz and Yam Levi and Zion English and Vikram Voleti and Adam Letts and Varun Jampani and Robin Rombach},
      year={2023},
      eprint={2311.15127},
      archivePrefix={arXiv},
      primaryClass={cs.CV},
      url={https://arxiv.org/abs/2311.15127}, 
}

@misc{liu2024sorareviewbackgroundtechnology,
      title={Sora: A Review on Background, Technology, Limitations, and Opportunities of Large Vision Models}, 
      author={Yixin Liu and Kai Zhang and Yuan Li and Zhiling Yan and Chujie Gao and Ruoxi Chen and Zhengqing Yuan and Yue Huang and Hanchi Sun and Jianfeng Gao and Lifang He and Lichao Sun},
      year={2024},
      eprint={2402.17177},
      archivePrefix={arXiv},
      primaryClass={cs.CV},
      url={https://arxiv.org/abs/2402.17177}, 
}

@misc{singer2022makeavideotexttovideogenerationtextvideo,
      title={Make-A-Video: Text-to-Video Generation without Text-Video Data}, 
      author={Uriel Singer and Adam Polyak and Thomas Hayes and Xi Yin and Jie An and Songyang Zhang and Qiyuan Hu and Harry Yang and Oron Ashual and Oran Gafni and Devi Parikh and Sonal Gupta and Yaniv Taigman},
      year={2022},
      eprint={2209.14792},
      archivePrefix={arXiv},
      primaryClass={cs.CV},
      url={https://arxiv.org/abs/2209.14792}, 
}

@misc{ho2022classifierfreediffusionguidance,
      title={Classifier-Free Diffusion Guidance}, 
      author={Jonathan Ho and Tim Salimans},
      year={2022},
      eprint={2207.12598},
      archivePrefix={arXiv},
      primaryClass={cs.LG},
      url={https://arxiv.org/abs/2207.12598}, 
}

@misc{karras2022elucidatingdesignspacediffusionbased,
      title={Elucidating the Design Space of Diffusion-Based Generative Models}, 
      author={Tero Karras and Miika Aittala and Timo Aila and Samuli Laine},
      year={2022},
      eprint={2206.00364},
      archivePrefix={arXiv},
      primaryClass={cs.CV},
      url={https://arxiv.org/abs/2206.00364}, 
}

@misc{chung2024cfgmanifoldconstrainedclassifierfree,
      title={CFG++: Manifold-constrained Classifier Free Guidance for Diffusion Models}, 
      author={Hyungjin Chung and Jeongsol Kim and Geon Yeong Park and Hyelin Nam and Jong Chul Ye},
      year={2024},
      eprint={2406.08070},
      archivePrefix={arXiv},
      primaryClass={cs.CV},
      url={https://arxiv.org/abs/2406.08070}, 
}

@misc{liu2023compositionalvisualgenerationcomposable,
      title={Compositional Visual Generation with Composable Diffusion Models}, 
      author={Nan Liu and Shuang Li and Yilun Du and Antonio Torralba and Joshua B. Tenenbaum},
      year={2023},
      eprint={2206.01714},
      archivePrefix={arXiv},
      primaryClass={cs.CV},
      url={https://arxiv.org/abs/2206.01714}, 
}

@misc{bai2024zigzagdiffusionsamplingdiffusion,
      title={Zigzag Diffusion Sampling: Diffusion Models Can Self-Improve via Self-Reflection}, 
      author={Lichen Bai and Shitong Shao and Zikai Zhou and Zipeng Qi and Zhiqiang Xu and Haoyi Xiong and Zeke Xie},
      year={2024},
      eprint={2412.10891},
      archivePrefix={arXiv},
      primaryClass={cs.CV},
      url={https://arxiv.org/abs/2412.10891}, 
}

@misc{xu2023restartsamplingimprovinggenerative,
      title={Restart Sampling for Improving Generative Processes}, 
      author={Yilun Xu and Mingyang Deng and Xiang Cheng and Yonglong Tian and Ziming Liu and Tommi Jaakkola},
      year={2023},
      eprint={2306.14878},
      archivePrefix={arXiv},
      primaryClass={cs.LG},
      url={https://arxiv.org/abs/2306.14878}, 
}

@misc{lugmayr2022repaintinpaintingusingdenoising,
      title={RePaint: Inpainting using Denoising Diffusion Probabilistic Models}, 
      author={Andreas Lugmayr and Martin Danelljan and Andres Romero and Fisher Yu and Radu Timofte and Luc Van Gool},
      year={2022},
      eprint={2201.09865},
      archivePrefix={arXiv},
      primaryClass={cs.CV},
      url={https://arxiv.org/abs/2201.09865}, 
}

@misc{lu2022dpmsolverfastodesolver,
      title={DPM-Solver: A Fast ODE Solver for Diffusion Probabilistic Model Sampling in Around 10 Steps}, 
      author={Cheng Lu and Yuhao Zhou and Fan Bao and Jianfei Chen and Chongxuan Li and Jun Zhu},
      year={2022},
      eprint={2206.00927},
      archivePrefix={arXiv},
      primaryClass={cs.LG},
      url={https://arxiv.org/abs/2206.00927}, 
}

@article{Lu_2025,
   title={DPM-Solver++: Fast Solver for Guided Sampling of Diffusion Probabilistic Models},
   volume={22},
   ISSN={2731-5398},
   url={http://dx.doi.org/10.1007/s11633-025-1562-4},
   DOI={10.1007/s11633-025-1562-4},
   number={4},
   journal={Machine Intelligence Research},
   publisher={Springer Science and Business Media LLC},
   author={Lu, Cheng and Zhou, Yuhao and Bao, Fan and Chen, Jianfei and Li, Chongxuan and Zhu, Jun},
   year={2025},
   month=jun, pages={730–751} }

@misc{lin2024commondiffusionnoiseschedules,
      title={Common Diffusion Noise Schedules and Sample Steps are Flawed}, 
      author={Shanchuan Lin and Bingchen Liu and Jiashi Li and Xiao Yang},
      year={2024},
      eprint={2305.08891},
      archivePrefix={arXiv},
      primaryClass={cs.CV},
      url={https://arxiv.org/abs/2305.08891}, 
}

@misc{saharia2022photorealistictexttoimagediffusionmodels,
      title={Photorealistic Text-to-Image Diffusion Models with Deep Language Understanding}, 
      author={Chitwan Saharia and William Chan and Saurabh Saxena and Lala Li and Jay Whang and Emily Denton and Seyed Kamyar Seyed Ghasemipour and Burcu Karagol Ayan and S. Sara Mahdavi and Rapha Gontijo Lopes and Tim Salimans and Jonathan Ho and David J Fleet and Mohammad Norouzi},
      year={2022},
      eprint={2205.11487},
      archivePrefix={arXiv},
      primaryClass={cs.CV},
      url={https://arxiv.org/abs/2205.11487}, 
}

@misc{ju2023directinversionboostingdiffusionbased,
      title={Direct Inversion: Boosting Diffusion-based Editing with 3 Lines of Code}, 
      author={Xuan Ju and Ailing Zeng and Yuxuan Bian and Shaoteng Liu and Qiang Xu},
      year={2023},
      eprint={2310.01506},
      archivePrefix={arXiv},
      primaryClass={cs.CV},
      url={https://arxiv.org/abs/2310.01506}, 
}

@misc{zhao2023unipcunifiedpredictorcorrectorframework,
      title={UniPC: A Unified Predictor-Corrector Framework for Fast Sampling of Diffusion Models}, 
      author={Wenliang Zhao and Lujia Bai and Yongming Rao and Jie Zhou and Jiwen Lu},
      year={2023},
      eprint={2302.04867},
      archivePrefix={arXiv},
      primaryClass={cs.LG},
      url={https://arxiv.org/abs/2302.04867}, 
}

@Inbook{Hairer2006,
author="Hairer, Ernst
and Wanner, Gerhard
and Lubich, Christian",
title="Backward Error Analysis and Structure Preservation",
bookTitle="Geometric Numerical Integration: Structure-Preserving Algorithms for Ordinary Differential Equations",
year="2006",
publisher="Springer Berlin Heidelberg",
address="Berlin, Heidelberg",
pages="337--388",
isbn="978-3-540-30666-5",
doi="10.1007/3-540-30666-8_9",
url="https://doi.org/10.1007/3-540-30666-8_9"
}

@misc{li2024hunyuandit,
      title={Hunyuan-DiT: A Powerful Multi-Resolution Diffusion Transformer with Fine-Grained Chinese Understanding}, 
      author={Zhimin Li and Jianwei Zhang and Qin Lin and Jiangfeng Xiong and Yanxin Long and Xinchi Deng and Yingfang Zhang and Xingchao Liu and Minbin Huang and Zedong Xiao and Dayou Chen and Jiajun He and Jiahao Li and Wenyue Li and Chen Zhang and Rongwei Quan and Jianxiang Lu and Jiabin Huang and Xiaoyan Yuan and Xiaoxiao Zheng and Yixuan Li and Jihong Zhang and Chao Zhang and Meng Chen and Jie Liu and Zheng Fang and Weiyan Wang and Jinbao Xue and Yangyu Tao and Jianchen Zhu and Kai Liu and Sihuan Lin and Yifu Sun and Yun Li and Dongdong Wang and Mingtao Chen and Zhichao Hu and Xiao Xiao and Yan Chen and Yuhong Liu and Wei Liu and Di Wang and Yong Yang and Jie Jiang and Qinglin Lu},
      year={2024},
      eprint={2405.08748},
      archivePrefix={arXiv},
      primaryClass={cs.CV}
}

@article{yang2024cogvideox,
  title={CogVideoX: Text-to-Video Diffusion Models with An Expert Transformer},
  author={Yang, Zhuoyi and Teng, Jiayan and Zheng, Wendi and Ding, Ming and Huang, Shiyu and Xu, Jiazheng and Yang, Yuanming and Hong, Wenyi and Zhang, Xiaohan and Feng, Guanyu and others},
  journal={arXiv preprint arXiv:2408.06072},
  year={2024}
}

@misc{wallace2023diffusionmodelalignmentusing,
      title={Diffusion Model Alignment Using Direct Preference Optimization}, 
      author={Bram Wallace and Meihua Dang and Rafael Rafailov and Linqi Zhou and Aaron Lou and Senthil Purushwalkam and Stefano Ermon and Caiming Xiong and Shafiq Joty and Nikhil Naik},
      year={2023},
      eprint={2311.12908},
      archivePrefix={arXiv},
      primaryClass={cs.CV},
      url={https://arxiv.org/abs/2311.12908}, 
}

@misc{sabour2024alignstepsoptimizingsampling,
      title={Align Your Steps: Optimizing Sampling Schedules in Diffusion Models}, 
      author={Amirmojtaba Sabour and Sanja Fidler and Karsten Kreis},
      year={2024},
      eprint={2404.14507},
      archivePrefix={arXiv},
      primaryClass={cs.CV},
      url={https://arxiv.org/abs/2404.14507}, 
}

@article{wang2023modelscope,
      title={Modelscope text-to-video technical report},
      author={Wang, Jiuniu and Yuan, Hangjie and Chen, Dayou and Zhang, Yingya and Wang, Xiang and Zhang, Shiwei},
      journal={arXiv preprint arXiv:2308.06571},
      year={2023}
    }

@misc{kingma2023variationaldiffusionmodels,
      title={Variational Diffusion Models}, 
      author={Diederik P. Kingma and Tim Salimans and Ben Poole and Jonathan Ho},
      year={2023},
      eprint={2107.00630},
      archivePrefix={arXiv},
      primaryClass={cs.LG},
      url={https://arxiv.org/abs/2107.00630}, 
}

@misc{kirstain2023pickapicopendatasetuser,
      title={Pick-a-Pic: An Open Dataset of User Preferences for Text-to-Image Generation}, 
      author={Yuval Kirstain and Adam Polyak and Uriel Singer and Shahbuland Matiana and Joe Penna and Omer Levy},
      year={2023},
      eprint={2305.01569},
      archivePrefix={arXiv},
      primaryClass={cs.CV},
      url={https://arxiv.org/abs/2305.01569}, 
}

@misc{ghosh2023genevalobjectfocusedframeworkevaluating,
      title={GenEval: An Object-Focused Framework for Evaluating Text-to-Image Alignment}, 
      author={Dhruba Ghosh and Hanna Hajishirzi and Ludwig Schmidt},
      year={2023},
      eprint={2310.11513},
      archivePrefix={arXiv},
      primaryClass={cs.CV},
      url={https://arxiv.org/abs/2310.11513}, 
}

@article{yuan2024chronomagic,
  title={Chronomagic-bench: A benchmark for metamorphic evaluation of text-to-time-lapse video generation},
  author={Yuan, Shenghai and Huang, Jinfa and Xu, Yongqi and Liu, Yaoyang and Zhang, Shaofeng and Shi, Yujun and Zhu, Rui-Jie and Cheng, Xinhua and Luo, Jiebo and Yuan, Li},
  journal={Advances in Neural Information Processing Systems},
  volume={37},
  pages={21236--21270},
  year={2024}
}

@article{wu2023human,
  title={Human Preference Score v2: A Solid Benchmark for Evaluating Human Preferences of Text-to-Image Synthesis},
  author={Wu, Xiaoshi and Hao, Yiming and Sun, Keqiang and Chen, Yixiong and Zhu, Feng and Zhao, Rui and Li, Hongsheng},
  journal={arXiv preprint arXiv:2306.09341},
  year={2023}
}

@misc{xu2023imagerewardlearningevaluatinghuman,
      title={ImageReward: Learning and Evaluating Human Preferences for Text-to-Image Generation}, 
      author={Jiazheng Xu and Xiao Liu and Yuchen Wu and Yuxuan Tong and Qinkai Li and Ming Ding and Jie Tang and Yuxiao Dong},
      year={2023},
      eprint={2304.05977},
      archivePrefix={arXiv},
      primaryClass={cs.CV},
      url={https://arxiv.org/abs/2304.05977}, 
}

@misc{schuhmann2022laion5bopenlargescaledataset,
      title={LAION-5B: An open large-scale dataset for training next generation image-text models}, 
      author={Christoph Schuhmann and Romain Beaumont and Richard Vencu and Cade Gordon and Ross Wightman and Mehdi Cherti and Theo Coombes and Aarush Katta and Clayton Mullis and Mitchell Wortsman and Patrick Schramowski and Srivatsa Kundurthy and Katherine Crowson and Ludwig Schmidt and Robert Kaczmarczyk and Jenia Jitsev},
      year={2022},
      eprint={2210.08402},
      archivePrefix={arXiv},
      primaryClass={cs.CV},
      url={https://arxiv.org/abs/2210.08402}, 
}

@misc{radford2021learningtransferablevisualmodels,
      title={Learning Transferable Visual Models From Natural Language Supervision}, 
      author={Alec Radford and Jong Wook Kim and Chris Hallacy and Aditya Ramesh and Gabriel Goh and Sandhini Agarwal and Girish Sastry and Amanda Askell and Pamela Mishkin and Jack Clark and Gretchen Krueger and Ilya Sutskever},
      year={2021},
      eprint={2103.00020},
      archivePrefix={arXiv},
      primaryClass={cs.CV},
      url={https://arxiv.org/abs/2103.00020}, 
}

@inproceedings{zhang2018perceptual,
  title={The Unreasonable Effectiveness of Deep Features as a Perceptual Metric},
  author={Zhang, Richard and Isola, Phillip and Efros, Alexei A and Shechtman, Eli and Wang, Oliver},
  booktitle={CVPR},
  year={2018}
}

@misc{nilsson2020understandingssim,
      title={Understanding SSIM}, 
      author={Jim Nilsson and Tomas Akenine-Möller},
      year={2020},
      eprint={2006.13846},
      archivePrefix={arXiv},
      primaryClass={eess.IV},
      url={https://arxiv.org/abs/2006.13846}, 
}

@misc{sauer2023adversarialdiffusiondistillation,
      title={Adversarial Diffusion Distillation}, 
      author={Axel Sauer and Dominik Lorenz and Andreas Blattmann and Robin Rombach},
      year={2023},
      eprint={2311.17042},
      archivePrefix={arXiv},
      primaryClass={cs.CV},
      url={https://arxiv.org/abs/2311.17042}, 
}

@misc{xu2025goodseedmakesgood,
      title={Good Seed Makes a Good Crop: Discovering Secret Seeds in Text-to-Image Diffusion Models}, 
      author={Katherine Xu and Lingzhi Zhang and Jianbo Shi},
      year={2025},
      eprint={2405.14828},
      archivePrefix={arXiv},
      primaryClass={cs.CV},
      url={https://arxiv.org/abs/2405.14828}, 
}

@misc{samuel2023generatingimagesrareconcepts,
      title={Generating images of rare concepts using pre-trained diffusion models}, 
      author={Dvir Samuel and Rami Ben-Ari and Simon Raviv and Nir Darshan and Gal Chechik},
      year={2023},
      eprint={2304.14530},
      archivePrefix={arXiv},
      primaryClass={cs.CV},
      url={https://arxiv.org/abs/2304.14530}, 
}

@misc{mao2024lotterytickethypothesisdenoising,
      title={The Lottery Ticket Hypothesis in Denoising: Towards Semantic-Driven Initialization}, 
      author={Jiafeng Mao and Xueting Wang and Kiyoharu Aizawa},
      year={2024},
      eprint={2312.08872},
      archivePrefix={arXiv},
      primaryClass={cs.CV},
      url={https://arxiv.org/abs/2312.08872}, 
}

@misc{zhou2025goldennoisediffusionmodels,
      title={Golden Noise for Diffusion Models: A Learning Framework}, 
      author={Zikai Zhou and Shitong Shao and Lichen Bai and Shufei Zhang and Zhiqiang Xu and Bo Han and Zeke Xie},
      year={2025},
      eprint={2411.09502},
      archivePrefix={arXiv},
      primaryClass={cs.LG},
      url={https://arxiv.org/abs/2411.09502}, 
}

@misc{poyuan2023syntheticshiftsinitialseed,
      title={Synthetic Shifts to Initial Seed Vector Exposes the Brittle Nature of Latent-Based Diffusion Models}, 
      author={Mao Po-Yuan and Shashank Kotyan and Tham Yik Foong and Danilo Vasconcellos Vargas},
      year={2023},
      eprint={2312.11473},
      archivePrefix={arXiv},
      primaryClass={cs.CV},
      url={https://arxiv.org/abs/2312.11473}, 
}

@misc{yang2025texttoimagerectifiedflowplugandplay,
      title={Text-to-Image Rectified Flow as Plug-and-Play Priors}, 
      author={Xiaofeng Yang and Cheng Chen and Xulei Yang and Fayao Liu and Guosheng Lin},
      year={2025},
      eprint={2406.03293},
      archivePrefix={arXiv},
      primaryClass={cs.CV},
      url={https://arxiv.org/abs/2406.03293}, 
}

@misc{wu2023tuneavideooneshottuningimage,
      title={Tune-A-Video: One-Shot Tuning of Image Diffusion Models for Text-to-Video Generation}, 
      author={Jay Zhangjie Wu and Yixiao Ge and Xintao Wang and Weixian Lei and Yuchao Gu and Yufei Shi and Wynne Hsu and Ying Shan and Xiaohu Qie and Mike Zheng Shou},
      year={2023},
      eprint={2212.11565},
      archivePrefix={arXiv},
      primaryClass={cs.CV},
      url={https://arxiv.org/abs/2212.11565}, 
}

@misc{shen2024rethinkingspatialinconsistencyclassifierfree,
      title={Rethinking the Spatial Inconsistency in Classifier-Free Diffusion Guidance}, 
      author={Dazhong Shen and Guanglu Song and Zeyue Xue and Fu-Yun Wang and Yu Liu},
      year={2024},
      eprint={2404.05384},
      archivePrefix={arXiv},
      primaryClass={cs.CV},
      url={https://arxiv.org/abs/2404.05384}, 
}

@misc{hong2024smoothedenergyguidanceguiding,
      title={Smoothed Energy Guidance: Guiding Diffusion Models with Reduced Energy Curvature of Attention}, 
      author={Susung Hong},
      year={2024},
      eprint={2408.00760},
      archivePrefix={arXiv},
      primaryClass={cs.CV},
      url={https://arxiv.org/abs/2408.00760}, 
}

@inproceedings{chen_2024, series={MM ’24},
   title={SATO: Stable Text-to-Motion Framework},
   url={http://dx.doi.org/10.1145/3664647.3681034},
   DOI={10.1145/3664647.3681034},
   booktitle={Proceedings of the 32nd ACM International Conference on Multimedia},
   publisher={ACM},
   author={chen, Wenshuo and Xiao, Hongru and Zhang, Erhang and Hu, Lijie and Wang, Lei and Liu, Mengyuan and Chen, Chen},
   year={2024},
   month=Oct, pages={6989–6997},
   collection={MM ’24} }

@misc{ning2025dctdiffintriguingpropertiesimage,
      title={DCTdiff: Intriguing Properties of Image Generative Modeling in the DCT Space}, 
      author={Mang Ning and Mingxiao Li and Jianlin Su and Haozhe Jia and Lanmiao Liu and Martin Beneš and Wenshuo Chen and Albert Ali Salah and Itir Onal Ertugrul},
      year={2025},
      eprint={2412.15032},
      archivePrefix={arXiv},
      primaryClass={cs.CV},
      url={https://arxiv.org/abs/2412.15032}, 
}

@misc{jia2025physicsinformedrepresentationalignmentsparse,
      title={Physics-Informed Representation Alignment for Sparse Radio-Map Reconstruction}, 
      author={Haozhe Jia and Wenshuo Chen and Zhihui Huang and Lei Wang and Hongru Xiao and Nanqian Jia and Keming Wu and Songning Lai and Bowen Tian and Yutao Yue},
      year={2025},
      eprint={2501.19160},
      archivePrefix={arXiv},
      primaryClass={cs.CV},
      url={https://arxiv.org/abs/2501.19160}, 
}

@inproceedings{Chen_2025, series={MM ’25},
   title={ANT: Adaptive Neural Temporal-Aware Text-to-Motion Model},
   url={http://dx.doi.org/10.1145/3746027.3755168},
   DOI={10.1145/3746027.3755168},
   booktitle={Proceedings of the 33rd ACM International Conference on Multimedia},
   publisher={ACM},
   author={Chen, Wenshuo and Yu, Kuimou and Haozhe, Jia and Yuan, Kaishen and Huang, Zexu and Tian, Bowen and Lai, Songning and Xiao, Hongru and Zhang, Erhang and Wang, Lei and Yue, Yutao},
   year={2025},
   month=Oct, pages={9852–9861},
   collection={MM ’25} }

@misc{chen2025polarisprojectionorthogonalsquaresrobust,
      title={POLARIS: Projection-Orthogonal Least Squares for Robust and Adaptive Inversion in Diffusion Models}, 
      author={Wenshuo Chen and Haosen Li and Shaofeng Liang and Lei Wang and Haozhe Jia and Kaishen Yuan and Jieming Wu and Bowen Tian and Yutao Yue},
      year={2025},
      eprint={2512.00369},
      archivePrefix={arXiv},
      primaryClass={cs.CV},
      url={https://arxiv.org/abs/2512.00369}, 
}

@misc{yuan2025coemogensemanticallycoherentscalableemotional,
      title={CoEmoGen: Towards Semantically-Coherent and Scalable Emotional Image Content Generation}, 
      author={Kaishen Yuan and Yuting Zhang and Shang Gao and Yijie Zhu and Wenshuo Chen and Yutao Yue},
      year={2025},
      eprint={2508.03535},
      archivePrefix={arXiv},
      primaryClass={cs.CV},
      url={https://arxiv.org/abs/2508.03535}, 
}
\end{document}